\documentclass{tlp}
\usepackage{color}

\usepackage{times}
\usepackage{helvet}
\usepackage{courier}\title{Title}
\usepackage{aopmath}

\usepackage{verbatim}
\usepackage{mathptmx}  
\usepackage{amssymb}
\usepackage{graphicx}
\usepackage{url}

\newtheorem{definition}{Definition}
\newtheorem{remark}{Remark}

\def\at{\mathcal{A}}

\def\mr{\mathcal{R}}
\def\mp{\mathcal{P}}
\def\mf{\mathcal{F}}

\def\ar{\leftarrow}

\def\lrar{\Longrightarrow}
\def\beq{\begin{equation}}
\def\eeq#1{\label{#1}\end{equation}}
\def\ba{\begin{array}}
\def\ea{\end{array}}

\def\prop#1#2#3{

\noindent
{\bf Proposition~\ref{#1}{#2}} {\it{#3}}}

\def\gus{GUS}
\def\i#1{\hbox{\it #1\/}}
\def\is#1{\hbox{\scriptsize\it #1\/}}

\def\ar{\leftarrow}
\def\bc{\begin{center}}
\def\ec{\end{center}}
\def\smodels{{\sc smodels}}

\def\cmodels{{\sc cmodels}}

\def\minisatid{{\sc mini\-sat(id)}}

\def\clasp{{\sc clasp}}

\def\lparse{{\sc lparse}}
\def\gringo{{\sc gringo}}

\newcommand{\dpll}{\textsc{dpll}}

\newcommand{\nt}{not\;}
\newcommand{\stt}{\;|\;}
\newcommand{\opi}{O^\Pi}

\newcommand{\ol}{\overline}

\newcommand{\rup}{\textit{Unit Propagate}}
\newcommand{\rupl}{\textit{Unit Propagate Learn}}

\newcommand{\rbj}{\textit{Backjump}}

\newcommand{\rl}{\textit{Learn}}

\newcommand{\dec}{{\Delta}}

\newcommand{\rd}{\textit{Decide}}
\newcommand{\rf}{\textit{Fail}}
\newcommand{\rb}{\textit{Backtrack}}

\newcommand{\runf}{\textit{Unfounded}}

\newcommand{\rarc}{\textit{All Rules Cancelled}}
\newcommand{\rbt}{\textit{Backchain True}}

\newcommand{\fail}{\textit{FailState}}

\newcommand{\smtasp}{\textit{\sc{sm(asp)}}}
\newcommand{\smtaspl}{\textit{\sc{sml(asp)}}}
\newcommand{\sm}{\textit{\sc{sm}}}

\newcommand{\dps}{\textit{\sc{dp}}}

\newcommand{\cmsm}{\textit{\sc{atleast}}}

\newcommand{\boldr}{{\bf R}}

\newcommand{\dbar}{{||}}
\newcommand{\pcid}{PC(ID)}
\newcommand{\foid}{{FO(ID)}}

\newif\ifshowrevision
\showrevisiontrue
\newcommand{\rev}[2]{{\ifshowrevision\color{black}\fi#2}}

\begin{document}



\title{
 Transition Systems  for Model Generators --- 
 A Unifying Approach
}

\author[Y. Lierler and M. Truszczynski]
{YULIYA LIERLER and MIROSLAW TRUSZCZYNSKI\\
Department of Computer Science, University of Kentucky, Lexington, 
KY 40506-0633, USA\\
\email{yuliya,mirek@cs.uky.edu} 
}

\date{}


\maketitle
\begin{abstract}
A fundamental 
task for propositional logic is to compute 
models of propositional formulas. 
Programs developed for this task are called satisfiability solvers. We
show that transition systems introduced by Nieuwenhuis, Oliveras, and 
Tinelli to model and analyze satisfiability solvers can be adapted for 
solvers developed for two other propositional formalisms: logic 
programming under the answer-set semantics, and the logic {\pcid}. 
We show that in each case the task of computing models can be seen 
as ``satisfiability modulo answer-set programming,'' where the goal is
to find a model of a theory that also is an answer set of a certain 
program. 
The unifying perspective we develop shows, in 
particular, that solvers  {\clasp} and {\minisatid} 
are 
closely related  despite being developed for different 
formalisms, one for answer-set programming and the latter for the 
logic {\pcid}.
\end{abstract}

\section{Introduction}

A fundamental reasoning task for propositional logic is to compute 
models of propositional formulas or determine that no models exist.
Programs developed for this task are commonly called \emph{model generators}
or \emph{satisfiability (SAT) solvers}. In the paper we show that transition 
systems introduced by Nieuwenhuis et al. \citeyear{nie06}
to model and analyze SAT solvers can be adapted for the analysis and
comparison of solvers developed for other propositional 
formalisms. The two formalisms we focus on are logic programming with 
the answer-set semantics and the logic {\pcid}. 

Davis-Putnam-Logemann-Loveland ({\dpll}) procedure is a well-known 
method 
that exhaustively explores 
interpretations to generate models of a propositional formula. 
Most modern SAT solvers are based
on variations of the {\dpll} procedure. Usually these variations are 
specified by pseudocode. Nieuwenhuis et al. \citeyear{nie06} proposed
an alternative approach based on the notion of a \emph{transition 
system} that describes ``states of computation'' and
allowed transitions between them. In this way, it defines 
a directed graph such that every execution of the {\dpll} procedure 
corresponds to a path in the graph. This abstract way of presenting 
{\dpll}-based algorithms simplifies the analysis of their correctness 
and facilitates studies of their properties --- instead of reasoning 
about pseudocode constructs, we reason about properties of a graph. 
For instance, by proving that the graph corresponding to a {\dpll}-based
algorithm is finite and acyclic we show that the algorithm always terminates.

Answer-set programming (ASP) \cite{mar99,nie99} is a declarative 
programming formalism based on the answer-set semantics of logic 
programs \cite{gel88}.  Generating answer sets of propositional
programs is the key step in  computation with ASP.
The logic \foid, introduced by Denecker \citeyear{den00}
is another 
formalism for declarative programming and knowledge representation. 
As in the case of ASP, most automated reasoning tasks 
in the logic FO(ID) reduce to reasoning in its propositional core, the 
logic PC(ID)~\cite{mar08}, where generating models is again the key.

In this paper, we show that both computing answer sets of programs
and computing models of {\pcid} theories can be considered as testing
\emph{satisfiability modulo theories} (SMT), where the objective is to
find a model of a set of clauses that is also an answer 
set of a certain program. We refer to this computational problem 
as \emph{satisfiability modulo answer-set programming} and denote it by SM(ASP).
We identify the propositional formalism capturing SM(ASP) --- we use the 
same term to refer to it --- and show that it is a common generalization 
of ASP and \pcid. We define a simple transition system for SM(ASP) and 
show that it can be used as an abstract representation of the solver 
{\smodels}\footnote{\tt  http://www.tcs.hut.fi/Software/smodels/ .}~\cite{nie00}, an alternative
to a similar characterization of {\smodels} obtained earlier by
Lierler~\citeyear{lier10}. We then define another 
more elaborate transition system for SM(ASP) that captures such features
of backtracking search as backjumping and learning. We use this transition
system to obtain 
abstract characterizations of the algorithms implemented by the ASP 
solvers {\cmodels}\footnote{\tt
  http://www.cs.utexas.edu/users/tag/cmodels .}~\cite{giu04a} and 
{\clasp}\footnote{\tt http://www.cs.uni-potsdam.de/clasp/ .}~\cite{geb07}, 
and the {\pcid} solver 
{\minisatid}\footnote{\tt http://dtai.cs.kuleuven.be/krr/software/minisatid
.}~\cite{mar08}.
Finally, we briefly mention the possibility to regard the introduced transition 
systems as proof systems. In that setting, transition systems could be
used for comparing the solvers they represent in terms of the complexity
of the corresponding proof systems.

Our results provide a uniform correctness proof for a 
broad class of solvers that can be modeled by the transition system for
SM(ASP), clarify essential computational principles behind ASP and 
{\pcid} solvers, and offer insights into how they relate to each other. 
In particular, our results yield the first abstract representation of {\clasp}
in terms of transition systems (up to now {\clasp} has been typically
specified in pseudocode), and show that at the abstract level,
{\clasp} and {\minisatid} are strikingly
closely related.

This last point is noteworthy as the two solvers were developed for 
different propositional formalisms.
{\minisatid} was developed specifically for the logic {\pcid}, where 
there is no concept of an answer set. The semantics is a natural 
extension of the notion of a model of a propositional theory to the 
setting when a theory consists of propositional clauses and 
\emph{definitions}. Definitions are written as logic programs but they
are interpreted by the well-founded semantics and not by
the answer-set semantics. 
There is no 
indication in the literature that {\clasp} or {\minisatid} were
influenced by each other. The two solvers were developed independently 
and for differently motivated formalisms. It is then of substantial
interest that at the level of solving they are closely related.

\section{Preliminaries}

We now review the abstract transition system framework
proposed for the {\dpll} procedure by Nieuwenhuis et al.
\citeyear{nie06}, and 
introduce some necessary terminology concerning logic programs and the
logic \pcid.

\smallskip
\noindent
{\bf Abstract DPLL.}
Most state-of-the-art SAT solvers are based on variations of the
{\dpll} procedure~\cite{dav62}. Nieuwenhuis et al. \citeyear{nie06}
described {\dpll} by means of a transition system that can be 
viewed as an abstract representation of the underlying {\dpll}
computation. 
In this section we review the abstract {\dpll} in the
form convenient for our purposes, following the presentation proposed
by Lierler~\citeyear{lier10}.

For a set $\at$ of atoms, 
a \emph{record} relative to~$\at$ is  
an ordered set $M$ of literals over $\at$, some possibly
annotated by $\dec$, which marks them as \emph{decision} literals.
A \emph{state} relative to~$\at$ is either a distinguished 
state {\fail} or a record relative to~$\at$. For instance, the
states relative to a singleton set~$\{ a\}$ 
are
\[
\ba{l}
{\fail},\ \ \emptyset,\ \ a,\ \ \neg a, \ \ a^\dec,\ \ \neg a^\dec,\ \ 
a~\neg a,\ \ a^\dec~\neg a,\\   
a~\neg a^\dec,\ \ a^\dec~\neg a^\dec,\ \ \neg a~a,\ \ \neg a^\dec~a,\ \ 
\neg a~a^\dec,\ \ \neg a^\dec~a^\dec\mbox{.}
 
\ea
\]
Frequently, we consider $M$ as a set of literals, 
ignoring both the annotations and the order among its elements. 
If neither a literal $l$ nor its dual, written $\overline{l}$, occurs 
in $M$, then $l$ is
\emph{unassigned} by $M$.
We say that $M$ is {\sl inconsistent} if both an atom $a$ and its
negation $\neg a$ occur in it.  For instance, states $b^\dec~\neg b$ and
$b~a~\neg b$ are inconsistent. 

If $C$ is a  disjunction (conjunction) of literals then by  
$\overline C$ we understand the
conjunction (disjunction) of the duals of
the literals occurring in $C$.
In some situations, we will identify disjunctions and conjunctions of
literals with the sets of these literals. 

In this paper, a \emph{clause} is a \emph{non-empty} disjunction of 
literals and a CNF formula is a conjunction (alternatively, a set) of clauses.
Each CNF formula $F$ determines its 
\emph{DPLL graph} $\dps_F$. 
The set of nodes of  $\dps_F$ consists of the states relative to the set of
atoms occurring in~$F$.
The edges of the graph $\dps_F$ are specified by four transition rules:
\[
\begin{array}[t]{lll}
\hbox{{\rup}:}&
M  ~\lrar~ 
       M~l &
  \hbox{if ~} 
 C\vee l\in F \hbox{ and ~}
 \ol C \subseteq M \\
\vspace{-1em}\\
\hbox{{\rd}:}& 
  M ~\lrar~ 
       M~l^\dec 
  & \hbox{if ~} l \mbox{~is unassigned by $M$} \\
\vspace{-1em}\\
\hbox{{\rf}:}&
  M ~\lrar~  {\fail}
  & \hbox{if ~} 
  \left\{ \begin{array}{l}
  \hbox{$M$ is inconsistent, and}\\
 \hbox{$M$ contains no decision literals}
  \end{array}\right. \\
\vspace{-1em}\\
\hbox{{\rb}:}&
P~l^\dec~Q\lrar~ 
  P~\overline{l}
  &\hbox{if ~} 
  \left\{ \begin{array}{l}
\hbox{$P~l^\dec~Q$ is inconsistent, and}\\  
\hbox{$Q$ contains no  decision literals.}
  \end{array}\right. \\
\end{array}
\]
A node (state) in the graph is \emph{terminal} if no edge originates in
it.  
The following proposition gathers key properties of 
the graph ${\dps}_F$. 

\begin{proposition}\label{prop:dp} For any CNF formula $F$,
\begin{itemize}
\item[(a)] graph ${\dps}_F$ is finite and acyclic,
\item[(b)] any terminal state of ${\dps}_F$ other than {\fail} 
is a model of $F$,
\item[(c)] {\fail} is reachable from $\emptyset$ in ${\dps}_F$ if and
  only if $F$ is unsatisfiable.
\end{itemize}
\end{proposition}

Thus, to decide the satisfiability of a CNF formula~$F$ it is enough to 
find a path leading from node $\emptyset$ to a terminal node $M$. If 
$M=\fail$, $F$ is unsatisfiable. Otherwise, $F$ is satisfiable and 
$M$ is a model of $F$.

For instance, let $F = \{a\vee b, \neg a\vee c\}$. Below we show a 
path in $\dps_F$ with every edge annotated by the name of the 
transition rule that gives rise to this edge in the graph:
\[
\emptyset\quad \stackrel{\rd}{\lrar}\quad a^\dec\quad \stackrel{\rup}{\lrar}\quad a^\dec~c\quad
\stackrel{\rd}{\lrar}\quad a^\dec~c~b^\dec\mbox{.}
\]
The state $a^\dec~c~b^\dec$ is terminal. Thus, 
Proposition~\ref{prop:dp}(b) asserts that $F$ is 
satisfiable and $\{a,c,b\}$ is a model of~$F$.

\smallskip
\noindent
{\bf Logic Programs.}
A \emph{(propositional) logic program} is a finite set of
 rules of the form
\beq
\ba {l}
a_0\ar a_1,\dots, a_l,not\ a_{l+1},\dots,not\ a_m,
not\ not\ a_{m+1},\dots,not\ not\ a_n,
\ea
\eeq{e:rule}
where $a_0$ is an atom or $\bot$ and each $a_i$, $1\leq i\leq n$, is
an atom.\footnote{In the paper, we do not use the term \emph{literal} for
expressions $a$, $not\ a$ and $not\ not\ a$. We reserve the term 
\emph{literal} exclusively for propositional literals $a$ and $\neg a$.} 
If $a_0$ is an atom then a 
rule~(\ref{e:rule}) is \emph{weakly normal}. If, in addition, $n=m$ 
then it is \emph{normal}. Programs 
consisting of weakly normal (normal, respectively) rules only are 
called \emph{weakly normal} (\emph{normal}, respectively). If $\Pi$ is
a program, by $At(\Pi)$ we denote the set of atoms that occur in $\Pi$. 

The expression $a_0$ is the \emph{head} of the rule. If $a_0=\bot$ we say
that the head of the rule is \emph{empty} and we often omit $\bot$ from 
the notation. In such case we require that $n>0$. We call a rule with the empty 
head a \emph{constraint}. 
We write $Head(\Pi)$ for the set of nonempty heads of
rules in a program $\Pi$.

We call the expression
$a_1,\dots, a_l,not\ a_{l+1},\dots,not\ a_m,\ not\ not\ a_{m+1},\dots,
not\ not\ a_n$
in a rule~(\ref{e:rule}) the \emph{body} of the rule and often view it
as the set of all elements that occur in it. 
If $a$ is an atom, we set $s(a)=s(not\; not\; a)=a$, and $s(not\; a)=\neg a$,
and we define $s(B)=\{s(l)\stt l\in B\}$. More directly, 
\[
s(B)=
  \{a_1,\ldots, a_l, \neg a_{l+1}, \ldots, \neg a_m,
a_{m+1}, \ldots, a_n\}\hbox{.}
\]
We also frequently identify the body $B$ of~(\ref{e:rule}) with the 
conjunction of elements in $s(B)$: 
\[
 a_1\wedge \cdots\wedge  a_l\wedge  \neg a_{l+1}\wedge \cdots\land\neg a_m
\wedge a_{m+1}\land\cdots \wedge a_n\hbox{.}
\]
By $Bodies(\Pi,a)$ 
we denote  the set of the bodies of all rules of~$\Pi$ with the head~$a$
(including the empty body). 
If $B$ is the body of~(\ref{e:rule}), we write $B^{pos}$ for the 
\emph{positive} part of the body, that is, $B^{pos}=\{a_1,\dots a_l\}$. 

We often interpret a rule~(\ref{e:rule}) as a propositional clause
\beq
a_0\vee \neg  a_1\vee \ldots\vee \neg a_l\vee   a_{l+1}\vee \ldots \vee
a_m\vee \neg  a_{m+1}\vee \ldots\vee \neg a_n
\eeq{e:cl}
(in the case when the rule is a constraint, $a_0$ is absent in (\ref{e:cl})).
Given a program $\Pi$, we write~$\Pi^{cl}$ for the set of clauses
(\ref{e:cl}) corresponding to all rules in $\Pi$.



This version of the language of logic programs is a special case of 
programs with nested expressions \cite{lif99d}. It is essential for 
our approach as it yields an alternative definition of the logic 
{\pcid}, which facilitates connecting it to ASP. We assume that the 
reader is familiar with the definition of an answer set of a logic 
program and refer to the paper by Lifschitz et al.~\citeyear{lif99d} 
for details.

\smallskip
\noindent
\textbf{Well-Founded Semantics and the Logic {\pcid}.}
Let $M$ be a set of (propositional) literals. By $\ol M$ we understand 
the set of the duals of the literals in $M$. A set $U$ 
of atoms occurring in a program~$\Pi$ is \emph{unfounded} 
on a consistent set $M$ of literals
with respect to $\Pi$ if for every $a\in U$ and every 
$B\in Bodies(\Pi,a)$, $M\cap \overline{s(B)}\not=\emptyset$
or 
$U\cap B^{pos}\neq\emptyset$. 
For every program $\Pi$ and for every consistent
set $M$ of literals, the union of sets that are unfounded on $M$ with 
respect to $\Pi$ is also unfounded on $M$ with respect to $\Pi$. Thus, 
under the assumptions above, there exists the \emph{greatest unfounded 
set} on $M$ with respect to $\Pi$. We denote this set by $\gus(M,\Pi)$.

For every weakly normal program $\Pi$ we define an operator $W_\Pi$ on 
a set $M$ of literals as follows
\[
W_\Pi(M)= \left\{ 
  \begin{array}{l l}
M\cup
\{a~\mid ~ a\ar B\in\Pi  \hbox{ and } s(B)\subseteq M\} 
\cup \ol{\gus(M,\Pi)}  & \text{if $M$ is consistent}\\
    At(\Pi)\cup\ol{At(\Pi)} &  \text{otherwise.}\\
   \end{array} \right.\\
\]
By $W^{fix}_\Pi(M)$ we denote a fixpoint of the operator $W_\Pi$ 
over a set $M$ of literals. One can show that it always 
exists since $W_\Pi$ is not only monotone but also increasing (for any
set~$M$ of literals, $M\subseteq W_\Pi(M)$).
The least fixpoint of 
$W_\Pi$, $W^{fix}_\Pi(\emptyset)$, is consistent and yields
the \emph{well-founded model} of $\Pi$, which in general is
three-valued. 
It is also written as $\hbox{\emph{lfp}}(W_\Pi)$.
These definitions and properties were initially introduced for normal 
programs only~\cite{van91}. They extend to programs in our syntax in a 
straightforward way, no changes in statements or arguments are needed
\cite{lee05}.

Let $\Pi$ be a program and $A$ a set of atoms. An atom $a$ is
\emph{open} with respect to
$\Pi$ and $A$ if \hbox{$a\in A\setminus Head(\Pi)$}.  We 
denote the set of atoms that are open with respect to
$\Pi$ and $A$ by $\opi_A$. 
By $\Pi_A$ we denote the logic program~$\Pi$ extended with the rules 
$a\ar \ not\ not\ a$ for each atom $a\in \opi_A$.
For instance, let $\Pi$ be a program
\beq
\ba l
a\ar b,\ not\ c\\
b\hbox{.}
\ea
\eeq{e:expi1}
Then, $\Pi_{\{c\}}$ is 
\[
\ba{l}
c\ar\ not\ not\ c\\ 
a\ar b,\ not\ c\\
b\hbox{.}
\ea
\]

We are ready to introduce the logic {\pcid} \cite{den00}.
%
A \emph{{\pcid} theory} is a pair~$(F,\Pi)$, where $F$ is a set of clauses and
$\Pi$ is a weakly normal logic program. For a {\pcid} theory~$(F,\Pi)$, 
by $\Pi^o$ we denote $\Pi_{At(F\cup \Pi)}$ and by $\opi$ we denote 
$\opi_{At(F\cup \Pi)}$ 
(where $At(F\cup \Pi)$ stands for the set of
atoms that occur in $F$ and $\Pi$). 
Moreover, for a set~$M$ of literals 
and a set $A$ of atoms, by $M^A$ we denote the set of those literals in~$M$ 
whose atoms occur in $A$. A set $M$ of literals is \emph{complete} over 
the set $At$ of atoms if every atom in $At$ occurs (possibly negated)
in $M$ and no other atoms occur in $M$.

\begin{definition}
Let $(F,\Pi)$ be a {\pcid} theory. A consistent and complete (over $At(F\cup
\Pi)$) set $M$ of literals 
is called a \emph{model} of $(F,\Pi)$ if
\begin{itemize}
\item[(i)] $M$
is a model of $F$, and
\item[(ii)] $M=W^{fix}_{\Pi^o}(M^{\opi})$.
\end{itemize}
\end{definition}

For instance, let $F$ be a clause $b\vee \neg c$ and $\Pi$ be 
 program~(\ref{e:expi1}). The {\pcid} theory $(F,\Pi)$ has two models
$\{b, \neg c, a\}$ and $\{b, c,$ $\neg a\}$. We note that although sets
$\{\neg b, \neg c, a\}$ and $\{\neg b, \neg c, \neg a\}$ 
 satisfy the condition (i), that is, are models of $F$, they do not 
satisfy the condition (ii) and therefore are not models of $(F,\Pi)$.

The introduced definition of a {\pcid} theory differs from the original 
one~\cite{den00}. Specifically, for us the second component of a 
{\pcid} theory is a weakly normal program rather than a set of 
normal programs (definitions).
Still, the two formalisms are closely related. 

\begin{proposition}\label{prop:defrel}
For a {\pcid} theory $(F,\Pi)$ such that  $\Pi$ is a normal
program, $M$ is a model of $(F,\Pi)$ if and only if $M$ is a model of
$(F,\{\Pi\})$ according to the definition in~\cite{den00}. 
\end{proposition}

As the restriction to a single program in {\pcid} theories is not
essential~\cite{mar08}, Proposition \ref{prop:defrel} shows that 
our definition of the logic {\pcid} can be regarded as a slight 
generalization of the original one (more general programs can appear 
as definitions in {\pcid} theories).

\section{Satisfiability Modulo ASP: a unifying framework for ASP 
and {\pcid} solvers}
\label{sec3}
 
For a theory $T$ the \emph{satisfiability modulo theory} (SMT) problem
is: given a formula~$F$, determine whether~$F$ is $T$-satisfiable, 
that is, whether there exists a model of~$F$ that is also a model 
of~$T$. We refer the reader to~\cite{nie06} for an introduction to 
SMT. Typically, a theory $T$ that defines a specific SMT problem is a 
first-order formula. The SMT problem that we consider here is
different. The theory $T$ is a logic program under the (slightly
modified) answer-set semantics. We  show that the resulting version 
of the SMT problem can be regarded as a joint extension of ASP and 
{\pcid}.

We start by describing the modification of the answer-set semantics 
that we have in mind. 

\begin{definition}
Given a logic program $\Pi$, a set $X$ of atoms is an \emph{input 
answer set} of $\Pi$ if $X$ is an answer set of $\Pi\cup
(X\setminus Head(\Pi))$. 
\end{definition}

Informally, the atoms of $X$ that cannot possibly be defined by $\Pi$ 
as they do not belong to $Head(\Pi)$ serve as ``input'' to $\Pi$. A set 
$X$ is an input answer set of $\Pi$ if it is an answer set of the program 
$\Pi$ extended with  these ``input'' atoms from $X$. 
Input answer sets are related to stable models of a propositional logic
program module~\cite{oik06}.

For instance, let us consider 
 program~(\ref{e:expi1}). Then, sets $\{b,c\}$, $\{a,b\}$ are input answer
 sets of the program whereas set $\{a,b,c\}$ is not.

There are two important cases when input answer sets of a program are
closely related to answer sets of the program.

\begin{proposition}
\label{prop:input}
For a logic program $\Pi$ and a set $X$ of atoms:
\begin{itemize}
\item[(a)] $X\subseteq Head(\Pi)$ and $X$ is an input answer set of~$\Pi$
if and only if $X$ is an answer set of~$\Pi$.  
\item[(b)] If $(X\setminus Head(\Pi))\cap At(\Pi) = \emptyset$, then 
$X$ is an input answer set of $\Pi$ if and only if $X\cap Head(\Pi)$ is 
an answer set of $\Pi$.
\end{itemize}
\end{proposition}

We  now introduce a propositional formalism that we call 
\emph{satisfiability modulo ASP} and denote by SM(ASP). Later in
the paper we show 
that SM(ASP) can be viewed as a common generalization of 
both ASP and \pcid. \emph{Theories} of SM(ASP) are pairs $[F,\Pi]$, 
where $F$ is a set of clauses and $\Pi$ is a program. In the definition
below and in the remainder of the paper, for a set $M$ of literals we
write $M^{+}$ to denote the set of atoms (non-negated literals) in
$M$. For instance, $\{a,\neg b\}^{+}=\{a\}$.
 
\begin{definition}
For an SM(ASP) theory $[F,\Pi]$, a consistent and
complete (over $At(F\cup\Pi)$) set $M$ of literals 
is a \emph{model} of $[F,\Pi]$
if $M$ is a model of $F$ and 
$M^{+}$ 
is an input answer set of $\Pi$.
\end{definition}

For instance, let $F$ be a clause $b\vee \neg c$ and $\Pi$ be 
 program~(\ref{e:expi1}). The SM(ASP) theory $[F,\Pi]$ has two models
$\{b, \neg c, a\}$ and $\{b, c, \neg a\}$. 

The problem of finding models of pairs $[F,\Pi]$ can be regarded as an 
SMT problem in which, given a formula $F$ and a program $\Pi$, the goal
is to find a model of $F$ that is (its representation by the set of 
its true atoms, to be precise) an input answer set of $\Pi$. This 
observation motivated our choice of the name for the formalism.

As for {\pcid} theories, also for an SM(ASP) theory $[F,\Pi]$ we write
$\Pi^o$ for the program~$\Pi_{At(\Pi\cup F)}$.
We have the following simple observation.
\begin{proposition}
\label{prop17}
A set $M$ of literals is a model of an SM(ASP) theory $[F,\Pi]$ 
if and only if $M$ is a model of an SM(ASP) theory $[F,\Pi^o]$.
\end{proposition}

It is evident that a set $M$ of literals is a model of $F$ if and only
if $M$ is a model of~$[F,\emptyset]$. Thus, SM(ASP) allows us to express
the propositional satisfiability problem. We now show that
the SM(ASP) formalism captures ASP.
Let $\Pi$ be a program. We say that a set $F$ of clauses is 
\emph{$\Pi$-safe} if
\begin{enumerate}
\item $F\models \neg a$, for every $a\in \opi_{At(\Pi)}$, and
\item for every answer set~$X$ of $\Pi$ there is a model $M$
of $F$ such that $X=M^{+}\cap Head(\Pi)$.
\end{enumerate}

\begin{proposition}
\label{prop:conn}
Let $\Pi$ be a program. For every $\Pi$-safe set $F$ of clauses, a set
$X$ of atoms is an answer set of~$\Pi$ if and only if
$X=M^{+}\cap At(\Pi)$, for some model $M$ of $[F,\Pi]$.
\end{proposition}

This result shows that for an appropriately chosen theory $F$, answer
sets of a program~$\Pi$  can be derived in a direct way from models of 
an SM(ASP) theory $[F,\Pi]$. 
There are several possible choices for $F$ that satisfy the requirement
of $\Pi$-safety. One of them is the Clark's
\emph{completion} of $\Pi$~\cite{cla78}. We recall that the completion of
a program $\Pi$ consists of clauses in $\Pi^{cl}$ and of the formulas
that can be written as
\beq
\neg a\vee \bigvee_{B\in Bodies(\Pi,a)} B
\eeq{f:comp}
for every atom $a$ in $\Pi$ that is not a fact (that is,
the set $Bodies(\Pi,a)$ contains 
no empty body). Formulas (\ref{f:comp}) can be clausified
in a straightforward way by applying distributivity. The set of all the
resulting clauses and of those in $\Pi^{cl}$ forms the \emph{clausified}
completion of $\Pi$, which we will denote by $Comp(\Pi)$.

The theory
$Comp(\Pi)$ does not involve any new atoms but it can be exponentially
larger than the completion formula before clausification. We can avoid
the exponential blow-up by introducing new atoms. Namely, for each body
$B$ of a rule in $\Pi$ with $|B|>1$, we introduce a fresh atom $f_B$. 
If $|B|=1$, then we define $f_B=s(l)$, where $l$ is the only element of 
$B$. 
By $\i{ED-Comp}(\Pi)$, we denote the set of the following clauses:
\begin{enumerate}
\item all clauses in $\Pi^{cl}$
\item all clauses $\neg a\vee \bigvee_{B\in Bodies(\Pi,a)} f_B$, \ for
every $a\in At(\Pi)$ such that $a$ is not a fact in $\Pi$ and $|Bodies(\Pi,a)|>1$ 
\item all clauses $\neg a\vee s(l)$, where $a\in At(\Pi)$,
$Bodies(\Pi,a)= \{B\}$ and $l\in B$,
\item all clauses $\neg a$, where $|Bodies(\Pi,a)|=0$
\item all clauses obtained by clausifying in the obvious way formulas
$f_B \leftrightarrow B$, where $B\in Bodies(\Pi,a)$, for some atom $a$
that is not a fact in $\Pi$ and $|Bodies(\Pi,a)|>1$.
\end{enumerate}
Clearly, the restrictions of models of the theory $\i{ED-Comp}(\Pi)$
to the original set of atoms are precisely the models of $Comp(\Pi)$
(and of the completion of $\Pi$). However, the size of $\i{ED-Comp}(\Pi)$
is linear in the size of $\Pi$. The theory $\i{ED-Comp}(\Pi)$ has long
been used in answer-set computation. 
Answer set solvers such as
{\cmodels}~\cite{giu04a} and {\clasp}~\cite{geb07} start their
computation by transforming the given program~$\Pi$ into $\i{ED-Comp}(\Pi)$. 

For instance, let  $\Pi$ be  
 program~(\ref{e:expi1}).
The completion of $\Pi$ is the formula 
\[
 (a\vee \neg b\vee c)\wedge b\wedge \neg c \wedge (\neg a \vee
 (b\wedge \neg c))\hbox{,}
\]
its clausified completion $\i{Comp}(\Pi)$ is the formula
\[
(a\vee\neg b\vee c)\wedge(\neg a\vee b)\wedge(\neg a\vee \neg c)\wedge b
\wedge \neg c\hbox{,}
\]
and, finally, $\i{ED-Comp}(\Pi)$ is the formula
\[
\begin{array}{l}
(a\vee\neg b\vee c)\wedge
(\neg a\vee f_{b\land \neg c})\wedge
(f_{b\land \neg c}\vee\neg b\vee c)\wedge\\
(\neg f_{b\land \neg c}\vee b)\wedge
(\neg f_{b\land \neg c}\vee \neg c) \wedge b \wedge \neg c\hbox{.}
\end{array}
\]



We now have the following corollary from Proposition \ref{prop:conn}.
\begin{corollary}
\label{cor:conn}
For a logic program $\Pi$ and  a set $X$ of atoms, the following
conditions are equivalent:
\begin{itemize}
\item[(a)] $X$ is an answer set of $\Pi$,
\item[(b)] $X=M^{+}$ for some model $M$ of the SM(ASP) theory
$[\{\neg a ~\mid ~ a\in \opi_{At(\Pi)}\},\Pi]$,
\item[(c)] $X=M^{+}$ for some model $M$ of the SM(ASP) theory  $[Comp(\Pi),\Pi]$,
\item[(d)] $X=M^{+}\cap At(\Pi)$
for some model $M$ of the SM(ASP) theory $[\i{ED-Comp}(\Pi),\Pi]$.
\end{itemize}
\end{corollary}

It is in this sense that ASP can be regarded as a fragment of SM(ASP).
Answer sets of a program $\Pi$ can be described in terms of models of
SM(ASP) theories. Moreover, answer-set computation can be reduced in a
straightforward way to the task of computing models of SM(ASP) theories.

\begin{remark}\label{r:1}
Corollary \ref{cor:conn} specifies three ways to  describe
answer sets of a program in terms of models of SM(ASP) theories. This 
offers an interesting view into answer-set generation. The CNF 
formulas appearing in the SM(ASP) theories in the conditions (b) - (d)
make explicit some of the ``propositional satisfiability inferences'' 
that may be used when computing answer sets. The condition (b)
shows that when computing answer sets of a program, 
atoms not occurring as heads can be inferred as false. The theory 
in (c) makes it clear that a much broader class 
of inferences can be used, namely those that are based on the clauses 
of the completion. The theory in (d) describes still additional inferences,
as now, thanks to new atoms, we can explicitly infer whether bodies of 
rules must evaluate to 
true or false. 
In each case, some inferences needed for generating answer sets are 
still not captured by the respective CNF theory and require a reference 
to the program $\Pi$. We note that it is 
possible to express these ``answer-set specific'' inferences in 
terms of clauses corresponding to loop formulas~\cite{lin04,lee05}. 
We do not consider this possibility in this paper.
\end{remark}

Next, we  show that SM(ASP)  encompasses the logic \pcid.
The well-founded model~$M$ of a program $\Pi$ is \emph{total} if it
assigns all atoms occurring in $\Pi$. For a {\pcid} theory $(F,\Pi)$, 
a program $\Pi$ is \emph{total on a model} $M$ of $F$  if 
$W^{fix}_{\Pi^o}(M^{\opi})$ is total. A program $\Pi$ is \emph{total} 
if~$\Pi$ is total on every model $M$ of $F$. The {\pcid}
theories $(F,\Pi)$ where~$\Pi$ is total form an important class of 
\emph{total} {\pcid} theories.

There is a tight relation between models of a total {\pcid} theory 
$(F,\Pi)$ and models of an SM(ASP) theory $[F,\Pi]$.

\begin{proposition}
\label{prop:pcidsmasp}
For a total {\pcid} theory $(F,\Pi)$
and  a set~$M$ of literals over
the set $At(F\cup \Pi)$ of atoms, the following conditions are
equivalent:
\begin{itemize}
\item[(a)] $M$ is a model of $(F,\Pi)$,
\item[(b)] $M$ is a model of the SM(ASP) theory
$[F,\Pi]$,
\item[(c)] $M$ is a model of the SM(ASP) theory
$[Comp(\Pi_{At(\Pi)})\cup F,\Pi]$,
\item[(d)] for some model $M'$ of the
SM(ASP) theory \hbox{$[\i{ED-Comp}(\Pi_{At(\Pi)})\cup F,\Pi]$},
$M=M'\cap At(F\cap \Pi)$.
\end{itemize}
\end{proposition}


The conditions (b), (c), (d) state that the logic {\pcid} restricted 
to total theories can be regarded as a fragment of the SM(ASP) 
formalism. The comments made in Remark~\ref{r:1} pertain also to
generation of models in the logic {\pcid}. 

We  now characterize models of SM(ASP) theories, and computations
that lead to them, in terms of transition systems. Later we  discuss 
implications this characterization has for ASP and {\pcid} solvers.


We define the transition graph $\smtasp_{F,\Pi}$ for an SM(ASP) theory 
$[F,\Pi]$ as follows. The set of nodes of the graph $\smtasp_{F,\Pi}$ 
consists of the states relative to $At(F\cup\Pi)$. There are five 
transition rules that characterize the edges of $\smtasp_{F,\Pi}$. The 
transition rules {\rup}, {\rd}, {\rf}, {\rb} of the graph 
${\dps}_{F\cup\Pi^{cl}}$, and the transition rule 
\[
\begin{array}[t]{ll}
\hbox{{\runf}: }&
M\ 
~\lrar~ 
  M~\neg a 
  \hbox{~ if }
\hbox{$a\in U$ for a set $U$ unfounded on $M$ w.r.t. $\Pi^o$}\hbox{.}\\
\end{array}
\]

The graph ${\smtasp}_{F,\Pi}$ can be used for deciding whether 
an SM(ASP) theory $[F,\Pi]$ has a 
model.
\begin{proposition}\label{prop:cm1} 
For any SM(ASP) theory $[F,\Pi]$,
\begin{itemize}
\item[(a)] graph ${\smtasp}_{F,\Pi}$ is finite and acyclic,
\item[(b)] for any terminal state $M$ of ${\smtasp}_{F,\Pi}$ other
  than {\fail}, $M$ is a model of~$[F,\Pi]$,
\item[(c)] {\fail} is reachable from $\emptyset$ in 
${\smtasp}_{F,\Pi}$  if and  only if $[F,\Pi]$ has no models.
\end{itemize}
\end{proposition}

Proposition \ref{prop:cm1} shows that algorithms that correctly
find a path in the graph ${\smtasp}_{F,\Pi}$ from $\emptyset$ to a 
terminal node can be regarded as SM(ASP) solvers. It also provides a proof
of correctness for every SM(ASP) solver  that can be shown to work in 
this way.

One of the ways in which SM(ASP) encompasses ASP (specifically, 
Corollary \ref{cor:conn}(c)) is closely related to the way the 
answer-set solver {\smodels} works. We recall that to represent 
{\smodels} Lierler~\citeyear{lier10} proposed a  graph $\sm_\Pi$.
We note that the rule {\runf}  above is closely 
related to the transition rule with the same name used in the definition
of $\sm_\Pi$~\cite{lier10}. In fact, if $\Pi=\Pi^o$ then these rules are identical.

Lierler~\citeyear{lier10} observed that {\smodels} as it is implemented 
never follows certain edges in the graph ${\sm}_\Pi$, and called such
edges \emph{singular}. Lierler~\citeyear{lier10} denoted by ${\sm}^{-}_\Pi$ 
the graph obtained by removing from $\sm_\Pi$ all its singular edges and
showed that  ${\sm}^{-}_\Pi$ is still sufficient to serve as an abstract 
model of a class of ASP solvers including 
{\smodels}. The concept of a singular edge extends literally to the case 
of the graph ${\smtasp}_{F,\Pi}$. An edge $M~\lrar~M'$ in the 
graph ${\smtasp}_{F,\Pi}$ is \emph{singular} if:
\begin{enumerate}
\item the only transition rule justifying this edge is {\runf}, and
\item some edge $M~\lrar~M''$ 
can be justified by a transition rule other than {\runf} or {\rd}.
\end{enumerate}
We define $\smtasp^{-}_{F,\Pi}$ as 
the graph obtained by removing all singular edges from ${\smtasp}_{F,\Pi}$. 
Proposition~\ref{prop:rel_sm_cm1} below can be seen as an extension of
Proposition~4 in~\cite{lier10} to non-tight programs.
 
\begin{proposition}\label{prop:rel_sm_cm1}
For every program $\Pi$, 
the graphs $\sm^{-}_\Pi$ and $\smtasp^{-}_{Comp(\Pi),\Pi}$ are equal.
\end{proposition}

It follows that 
the graph 
$\smtasp^{-}_{Comp(\Pi),\Pi}$ provides an abstract model of {\smodels}. 
We recall though that $Comp(\Pi)$ can be exponentially larger than the 
completion formula before clausification. Using ASP specific propagation 
rules such as {\rbt} and {\rarc}~\cite{lier10} allows {\smodels} to avoid explicit 
representation of the clausified completion and infer all the necessary 
transitions directly on the basis of the program $\Pi$.

A similar
relationship, in terms of pseudocode representations of {\smodels} and
{\dpll}, is established  in~\cite{giumar05} for tight programs. 

The answer-set solvers {\cmodels}, {\clasp} and the {\pcid} solver 
{\minisatid} cannot be described in terms of the graph ${\smtasp}$
nor its subgraphs. These solvers implement such advanced features of 
SAT and SMT solvers as learning (forgetting), backjumping and restarts 
(Nieuwenhuis et al. \citeyear{nie06} give a good overview of these 
techniques). In the next section we  \emph{extend} the graph 
${\smtasp}_{F,\Pi}$ with propagation rules that capture these 
techniques. In the subsequent section, we  discus how this new graph 
models solvers {\cmodels}, {\clasp}, and {\minisatid}. Then we
provide insights into how they are related.

\section{Backjumping and Learning for SM(ASP)}

Nieuwenhuis et al. (2006, Section~2.4) defined the \emph{DPLL System
with Learning} graph that can be used to describe most of the modern
SAT solvers, which typically implement such sophisticated techniques
as learning and backjumping. We  demonstrate how to extend these
findings to capture SM(ASP) framework with
learning and backjumping.

\rev{Let $\Pi$ be a program over a set $\at$ of atoms. We say that $\Pi$ 
\emph{entails} a formula $G$ over $\at$, written $\Pi\models G$,
if for every interpretation $N$ over $\at$ such that
\hbox{$N\cap At(\Pi)$} is an answer set of~$\Pi$, $N\models G$.  

Let $F$ and $G$ be formulas and $\Pi$ a program (in the language
determined by $\At$). We say that $F$ and $\Pi$ \emph{entail} $G$, 
written $F,\Pi\models G$, if for every interpretation $N$ such that 
$N\models F$ and $N\cap At(\Pi)$ is an answer set of~$\Pi$, we have 
$N\models G$. In most cases, we consider the entailment relations in
the context of an SM(ASP) theory $[F,\Pi]$. We then assume that $\at=
At(F\cup \Pi)$.}{Let $[F,\Pi]$ be an SM(ASP) theory and let $G$ be
a formula over $At(F\cup\Pi)$. We say that $[F,\Pi]$ \emph{entails} 
$G$, written $F,\Pi\models G$, if for every model $M$ of $[F,\Pi]$,
$M\models G$. 
}

For an SM(ASP) theory $[F,\Pi]$, an \emph{augmented state} relative 
to $F$ and $\Pi$ is either a distinguished 
state {\fail} or a pair $M||\Gamma$ where $M$ is a record 
 relative to the set of
atoms occurring in~$F$ and $\Pi$, and $\Gamma$ is a set of clauses over
$At(F\cup\Pi)$ such that $F,\Pi^o\models\Gamma$.

We now define a graph {\smtaspl}$_{F,\Pi}$ for  
an SM(ASP) theory $[F,\Pi]$. Its nodes
are the augmented states relative to $F$ and $\Pi$.
The  rules {\rd}, {\runf}, and {\rf} of  {\smtasp}$_{F,\Pi}$
 are extended to {\smtaspl}$_{F,\Pi}$ as
follows:   $M||\Gamma~\lrar~M'||\Gamma$ ($M||\Gamma~\lrar \fail$, respectively)
is an edge  in  {\smtaspl}$_{F,\Pi}$ justified
by {\rd} or {\runf} (\rf, respectively) if and only if
 $M~\lrar~M'$ ($M~\lrar~\fail$)
is an edge  in {\smtasp}$_{F,\Pi}$ justified
by {\rd} or {\runf} (\rf, respectively). 
The other transition rules of {\smtaspl}$_{F,\Pi}$ follow:
\[
\begin{array}[t]{ll}
\hbox{{\rupl}:}&
M||\Gamma  ~\lrar~ 
       M~l||\Gamma
  \hbox{~ if ~~~} 
  \left\{ \begin{array}{l} 
 C\vee l\in F \cup \Pi^{cl} \cup  \Gamma \hbox{ and ~}\\
 \ol C \subseteq M \\
  \end{array}\right. \\
\vspace{-1em}\\
\hbox{{\rbj}: }& 
P~l^\dec~Q||\Gamma\lrar~ 
  P~l'||\Gamma
  \hbox{~ if ~} 
  \left\{ \begin{array}{l}
\hbox{$P~l^\dec~Q$ is inconsistent and }\\  
F,\Pi^o\models l'\vee \ol P
  \end{array}\right. 
\\
\vspace{-1em}\\
\hbox{{\rl}:}&
M||\Gamma  ~\lrar~ 
       M||~ C,~\Gamma
  \hbox{~ if ~} 
\left\{ \begin{array}{l} 
\hbox{every atom in $C$ occurs in $F$ and}\\
\hbox{$F,\Pi^o\models C$.}\\
\end{array}\right. 
\\

\end{array}
\]
We refer to the transition rules \rupl, \runf,
\rbj, {\rd}, and {\rf} of the graph  $\smtaspl_{F,\Pi}$  as \emph{basic}.
We say that a node in the graph is \emph{semi-terminal} 
if  no rule other than {\rl} is applicable to it.
We  omit the  word ``augmented'' before  ``state'' when this is
clear from a context.

The graph $\smtaspl_{F,\Pi}$ can be used  for deciding whether an
SM(ASP) theory $[F,\Pi]$ has a model.

\begin{proposition}\label{prop:clasp} For any SM(ASP) theory $[F,\Pi]$,
\begin{itemize}
\item [(a)] 
every path in $\smtaspl_{F,\Pi}$ contains only finitely many edges
justified by basic transition rules,
\item [(b)] for any  semi-terminal state $M||\Gamma$  of
  $\smtaspl_{F,\Pi}$ reachable from $\emptyset||\emptyset$, $M$ is a
  model of $[F,\Pi]$,
\item [(c)] {\fail} is reachable from $\emptyset||\emptyset$ in
  $\smtaspl_{F,\Pi}$ if and  only if  $[F,\Pi]$ has no models.
\end{itemize}
\end{proposition}
On the one hand, Proposition~\ref{prop:clasp}~(a) asserts that if we 
construct a path from $\emptyset||\emptyset$ so that  basic transition 
rules periodically appear in it then some semi-terminal state is 
eventually reached. On the other hand, parts (b) and (c) of 
Proposition~\ref{prop:clasp} assert that as soon as a semi-terminal state 
is reached the problem of deciding whether $[F,\Pi]$ has a model is 
solved. In other words, Proposition~\ref{prop:clasp} shows that the 
graph $\smtaspl_{F,\Pi}$ gives rise to a class of correct algorithms
for computing models of an SM(ASP) theory $[F,\Pi]$. It gives a
proof of correctness  to every SM(ASP) solver in this class and a
proof of termination under the assumption that  basic transition 
rules periodically appear in a path constructed from 
$\emptyset||\emptyset$. 

Nieuwenhuis et al. \citeyear{nie06} proposed the transition rules 
to model such techniques as forgetting and restarts.
The graph $\smtaspl_{F,\Pi}$ can easily be extended with such rules.

\section{Abstract {\cmodels}, {\clasp} and {\minisatid}}

We can view a path in the graph ${\smtaspl}_{F,\Pi}$ 
as a description of a process of search for a model
of an SM(ASP) theory $[F,\Pi]$ by applying transition rules.
Therefore, we can characterize the algorithm
of a solver that utilizes the transition rules of ${\smtaspl}_{F,\Pi}$ 
by describing a strategy for choosing a path in this graph. 
A strategy can be based, in particular, on assigning priorities to
 transition rules of ${\smtaspl}_{F,\Pi}$, 
so that a solver  never applies a  rule in a state 
if  a rule with higher priority is applicable
to the same state.

We  use this approach to describe and compare the algorithms
implemented in the solvers {\cmodels}, {\clasp} and {\minisatid}.
We stress that we talk here about characterizing and comparing
algorithms and not their specific implementations in the solvers.
We refer to these algorithms as \emph{abstract} {\cmodels}, 
{\clasp} and {\minisatid}, respectively. 
Furthermore, we only discuss the abstract {\minisatid}
  for the case of the total {\pcid} theories whereas the {\minisatid} system
  implements additional {\sl totality check} propagation rule to deal
  with the non-total theories. 
Given a program~$\Pi$, abstract {\cmodels} and abstract {\clasp} 
construct first $\i{ED-Comp}(\Pi)$. Afterwards, they search the graph 
${\smtaspl}_{\is{ED-Comp}(\Pi),\Pi}$ for a path to a semi-terminal state. 
In other words, both algorithms, while in a node of the graph
${\smtaspl}_{\is{ED-Comp}(\Pi),\Pi}$, progress by selecting one of 
the outgoing edges. By Proposition \ref{prop:clasp} and Corollary 
\ref{cor:conn}, each algorithm is indeed a method to compute answer 
sets of programs.

However, abstract {\cmodels} selects edges according to the priorities
on the transition rules of the graph that are set as follows:
\[
\ba l
\rbj,\rf>>
\rup>>
\rd>>
\runf
\hbox{,}
\ea 
\]
while abstract {\clasp} uses a different prioritization:
\[
\ba l
\rbj,\rf>>
\rup>>
\runf>>
\rd
\hbox{.}
\ea 
\]
The  difference between the algorithms
boils down to when the rule $\runf$ is used.

We now describe the algorithm behind the {\pcid} solver 
{\minisatid}~\cite{mar08} for total {\pcid} theories --- the 
abstract {\minisatid}. Speaking precisely, {\minisatid} assumes that
the program $\Pi$ of the input {\pcid} theory $(F,\Pi)$ is in the 
\emph{definitional
normal form} \cite{mar09}. Therefore, in practice {\minisatid} is always 
used with a simple preprocessor that converts programs into the 
definitional normal form. 
We will assume here that this preprocessor is a part of {\minisatid}. Under 
this assumption, given a {\pcid} theory $(F,\Pi)$, {\minisatid} can be 
described as constructing the completion $\i{ED-Comp}(\Pi^o)$ (the new 
atoms are introduced by the preprocessor when it converts $\Pi$ into the 
definitional normal form, the completion part is performed by the 
{\minisatid} proper), and then uses the transitions of the graph 
${\smtaspl}_{\is{ED-Comp}(\Pi^o)\cup F,\Pi^o}$ to search for a path to a 
semi-terminal state. In other words, the graph 
${\smtaspl}_{\is{ED-Comp}(\Pi^o)\cup F,\Pi^o}$ represents the
abstract {\minisatid}.  The strategy used by the algorithm follows the 
prioritization:
\[
\ba l
\rbj,\rf>>
\rup>>
\runf>>
\rd
\hbox{.}
\ea
\]
By Propositions \ref{prop17} and \ref{prop:pcidsmasp},
the algorithm indeed computes 
models of total {\pcid} theories. 

Systems {\cmodels}, {\clasp}, and {\minisatid} implement conflict-driven 
backjumping and learning. They apply the transition rule {\rl} 
only when in a non-semi-terminal state reached by an application of {\rbj}. 
Thus, the rule {\rl} does not differentiate the algorithms and so we have 
not taken it into account when describing these algorithms.


\section{PC(ID) Theories as Logic Programs with Constraints}
\label{sec:6}

For a clause $C= \neg a_1\lor\ldots\lor\neg a_l \lor a_{l+1}\lor\ldots
\lor a_m$ we write $C^r$ to denote the corresponding rule constraint
\[
\leftarrow a_1,\ldots, a_l, \nt a_{l+1},\ldots, \nt a_m\hbox{.}
\]
For a set $F$ of clauses, we define $F^r=\{C^r \stt C\in F\}$. Finally, 
for a {\pcid} theory $(F,\Pi)$ we define a logic program 
$\pi(F,\Pi)$ by setting
\[
\pi(F,\Pi)=\Pi^o\cup F^r\mbox{.}
\]
The representation of a {\pcid} theory $(F,\Pi)$ as  $\pi(F,\Pi)$ is
similar to the translation of FO(ID) theories into logic programs 
with variables given by Mari{\"e}n et al.~\citeyear{maar04}. The difference 
is in the way atoms are ``opened.'' We do it by means of rules of the
form \hbox{$a \leftarrow \nt \nt a$}, while Mari{\"e}n et al. use pairs of
rules $a \leftarrow \nt a^*$ and $a^* \leftarrow \nt a$.

There is a close relation between models of a {\pcid} theory $(F,\Pi)$
and answer sets of a program $\pi(F,\Pi)$.

%
%
\begin{proposition}\label{pr:eqt}
For a total {\pcid} theory~$(F,\Pi)$ and a consistent and complete (over
$At(F\cup \Pi)$) set $M$ of literals,
 $M$ is a model of $(F,\Pi)$ if and only if $M^{+}$ is an answer set of $\pi(F,\Pi)$. 
\end{proposition}

%
%
%

A \emph{choice rule}
construct $\{a\}$~\cite{nie00} of 
the {\lparse}\footnote{\tt http://www.tcs.hut.fi/Software/smodels/ .} and 
{\gringo}\footnote{\tt http://potassco.sourceforge.net/ .} languages
 can be seen as an abbreviation for a rule
\hbox{$a\ar\ not\ not\ a$}~\cite{fer05b}.
Thus, in view of Proposition~\ref{pr:eqt}, 
any answer set solver implementing language of {\lparse} or {\gringo} 
is also a {\pcid} solver (an input total {\pcid} theory $(F,\Pi)$ 
needs to be translated into $\pi(F,\Pi)$).

The reduction implied by Proposition~\ref{pr:eqt} by itself does not 
show how to relate particular solvers. However, we recall that abstract 
{\minisatid} is captured by the graph 
${\smtaspl}_{{\is{ED-Comp}(\Pi^o)}\cup F,\Pi^o}$. Moreover, we have the
following property.

\begin{proposition}\label{prop:rel3}
For a {\pcid} theory $(F,\Pi)$, we have
\[
{\smtaspl}_{{\is{ED-Comp}(\pi(F,\Pi))},\pi(F,\Pi)}=
{\smtaspl}_{{\is{ED-Comp}(\Pi^o)}\cup F,\Pi^o}\hbox{.}
\]
\end{proposition}
The graph ${\smtaspl}_{{\is{ED-Comp}(\pi(F,\Pi))},\pi(F,\Pi)}$ captures
the way {\clasp} works on the program $\pi(F,\Pi)$. In addition, the 
{\minisatid} and {\clasp} algorithms use the same prioritization. Thus,
Proposition \ref{prop:rel3} implies that the abstract {\clasp} used as 
a {\pcid} solver coincides with the abstract {\minisatid}.

\section{Related Work and Discussion}

Lierler~\citeyear{lier10} introduced the graphs {\sc sml} and {\sc gtl}
that extended the graphs {\sc sm} and {\sc gt}~\cite{lier10}, respectively, with
transition rules {\rbj} and \rl.
The graph {\sc sml} was used to characterize the computation of such
answer set solvers implementing learning as {\sc smodels$_{cc}$}\footnote{\tt
	http://www.nku.edu/$\sim$wardj1/Research/smodels\_cc.html
  .}~\cite{smodelscc} and {\sc sup}\footnote{\tt http://www.cs.utexas.edu/users/tag/sup
  .}~\cite{lier10}
whereas the graph {\sc gtl} was used to characterize {\cmodels}. These
graphs are strongly
related to our graph {\smtaspl} but they are not appropriate for describing the computation behind 
answer set solver  {\clasp} or {\pcid} solver {\minisatid}. The graph 
{\sc sml} reflects only propagation steps based on a program whereas 
{\sc clasp} and {\sc minisat(id)} proceed by considering both the program
and a propositional theory. The graph {\sc gtl}, on the other hand, does 
not seem to provide a way to imitate the behavior of the {\runf} rule 
in the {\smtaspl} graph. 

Giunchiglia and Maratea~\citeyear{giumar05} studied the relation
between the answer set solver {\smodels} and the {\dpll} procedure
for the case of tight programs by means of pseudocode
analysis. Giunchiglia et al~\citeyear{giu08} continued this work by
comparing answer set solvers {\smodels}, {\sc dlv}\footnote{\tt
  http://www.dbai.tuwien.ac.at/proj/dlv/ .}~\cite{eit97}, and {\cmodels}
via  pseudocode. In this paper we use a different approach 
to relate solvers that was 
proposed by Lierler~\citeyear{lier10}. That is, we use
graphs to represent the algorithms implemented by solvers, and study 
the structure of these graphs to find how the corresponding solvers are
related. We use this
method to state the relation between the answer set solvers
{\cmodels}, {\clasp},
and the {\pcid} solver {\minisatid} designed for different knowledge
representation formalisms.

Gebser and Schaub~\citeyear{gebsch06iclp} introduced a deductive system 
for describing inferences involved in computing answer sets by tableaux
methods. The abstract framework presented in this paper
can be viewed as a deductive system also, but a very different one.
For instance, we describe backtracking and backjumping by inference
rule, while the Gebser-Schaub system does not. Also the Gebser-Schaub
system does not take learning into account.
Accordingly, the derivations considered in this paper describe a search 
process, while derivations in the Gebser-Schaub system do not. Further,
the abstract framework discussed here does not
have any inference rule similar to Cut; this is why its derivations are
paths rather than trees.

Mari{\"e}n~\citeyear{mar09} (Section~5.7) described a
{\sc MiniSat(ID)} transition system to model a computation behind
the {\pcid} solver  {\minisatid}. We recall that we modeled the abstract
{\minisatid} with the graph {\smtaspl}. The graphs {\smtaspl} and 
{\sc MiniSat(ID)} are defined using different sets of nodes and transition 
rules. For instance, {\smtaspl} allows  states containing inconsistent
sets
of literals whereas the {\sc MiniSat(ID)} graph considers consistent 
states only. Due to this difference the {\sc MiniSat(ID)} graph requires 
multiple versions of ``backjump'' and ``fail'' transition rules. 

We used transition systems to characterize algorithms for computing 
answer sets of logic programs and models of {\pcid} theories. 
These transition systems are also suitable for formal comparison of the 
strength or power of reasoning methods given  rules that specify them. 
An approach to do so was proposed by Mari{\"e}n~\citeyear{mar09} 
(Section~5.7), who introduced the concept of \emph{decide-efficiency}
for such analysis. We outline below how standard concepts of \emph{proof 
complexity}~\cite{coo79} can be adapted to the setting of
transition systems. 


Let $\at$ be an infinite set of atoms.
We
define a \emph{node over} $\at$ to be a symbol $\fail$, or a finite 
sequence of literals over~$\at$ with annotations.
For a propositional formalism $\mf$ over $\at$, a \emph{proof procedure} 
$\mp_\mf$ consists of graphs $G_T$, where $T$ ranges over all theories 
in $\mf$, such that for every theory $T$
(i) $G_T$ is composed of nodes over $\at$ and (ii)
$T$ is \emph{unsatisfiable} if and only if there is a path $p$ 
in $G_T$ from the empty (\emph{start}) node to the $\fail$ node.
We call each such path~$p$ \emph{a proof}. 
We say that a proof system 
$S$ is \emph{based} on a proof procedure $\mp_\mf$ if
(i) $S\subseteq \mf \times \mr$, where $\mr$ denotes the set of all
finite sequences of nodes over $\at$, and (ii)
$S(T,p)$ holds if and only if $p$ is a proof in the graph $G_T$ in
$\mp_\mf$.
Predicate $S$ is indeed a proof system in the sense of Cook \citeyear{coo79}
because (i) $S$ is polynomial-time
computable, and (ii)~$T$ is unsatisfiable if and only if there exists a 
proof $p$ such that $S(T,p)$ holds.
 
In this sense, each of the graphs (transition systems) we 
introduced in this paper can be regarded as a proof procedure for SM(ASP) 
(for those involving the rule {\rl}, under additional assumptions to 
ensure the rule can be efficiently implemented). Thus, transition systems 
determine proof systems. Consequently, they can be compared, as well as 
solvers that they capture, in terms of the complexity of the 
corresponding proof systems. 

\section{Conclusions}

In the paper, 
we proposed a formalism SM(ASP) that can be regarded as a common 
generalization of (clausal) propositional logic, ASP, and the logic 
{\pcid}. The formalism offers 
an elegant \emph{satisfiability modulo theories} perspective on the 
latter two. We present several characterizations of these formalisms 
in terms of SM(ASP) theories that differ in the explicitly identified 
``satisfiability'' component. 
Next, we proposed transition systems for SM(ASP) to provide abstract 
models of SM(ASP) model generators. The transition systems offer a clear
and uniform framework for describing model generation algorithms in 
SM(ASP). As SM(ASP) subsumes several propositional formalisms, such 
a uniform approach provides a general proof of correctness and termination
that applies to a broad class of model generators designed for these 
formalisms. It also allows us to describe in precise mathematical terms 
relations between algorithms designed for reasoning with different 
logics such as propositional logic, logic programming under answer-set
semantics and the logic {\pcid}, the latter two studied in detail in 
the paper. For instance, our results imply that at an abstract level of
transition systems, {\clasp} and {\minisatid} are essentially identical. 
Finally, we note that this work gives the first description of {\clasp}
in the abstract framework rather than in pseudocode. Such high level 
view on state-of-the-art solvers in different, yet, related 
propositional formalisms will further their understanding, and help port 
advances in solver technology from one area to another.

\section*{Acknowledgments}
We are grateful to Marc Denecker and Vladimir Lifschitz for useful
discussions. We are equally grateful
to the reviewers who helped eliminate minor technical problems and 
improve the presentation.
Yuliya Lierler was supported  by a CRA/NSF 2010 Computing Innovation 
Fellowship. Miroslaw Truszczynski was supported by the NSF grant IIS-0913459.


\newpage
\section*{Appendix: Proofs}

\subsection{Proof of Proposition~\ref{prop:defrel}}

We start with some additional notation and several lemmas.

Let $N$ be a set of literals. By $|N|$ we denote a set of atoms 
occurring in $N$. For instance $|\{a,~\neg b,~c\}|=\{a,~b,~c\}$.
Further, by $ch(N)$ we denote a set of rules of the form $a\ar\ not
\ not\ a$, where $a\in|N|$. 

By a \emph{program literal} we mean expressions $a$, $not\ a$ and 
$not\ not\ a$, where $a$ is an atom. For a program literal $l$,
we set $s(l)=a$, if $l= a$ or $l=not\ not\ a$, and $s(l)=\neg a$, if 
$l=not\ a$. For  a set $B$ of body literals, we define $s(B)=\{s(l)\stt
l\in B\}$.
If $\Pi$ is a program and $N$ is a set of literals, by 
$\Pi(N)$ we denote the program obtained from $\Pi$ by removing each rule
whose body contains a program literal $l$ such that $\overline{s(l)} \in
N$, and deleting from the bodies of all rules in $\Pi$ every program
literal $l$ such that $s(l)\in N$. 

\begin{lemma}\label{lem:w1}
Let $\Pi$ be a logic program and $N$ a consistent set of literals such 
that $|N|\cap Head(\Pi) =\emptyset$. For every consistent set $M$ of 
literals such that $|N|\cap |M|=\emptyset$,
$$
\{a\stt a\ar B\in\Pi\cup ch(N)  \hbox{ and } s(B) \subseteq M\cup N\}
\setminus N=
  \{a\stt a\ar B\in\Pi(N)  \hbox{ and } s(B)\subseteq N\}.
$$
\end{lemma}
\begin{proof}
Let $c\in \{a\stt a\ar B\in\Pi\cup ch(N) \hbox{ and } s(B)\subseteq
M\cup N\} \setminus N$. Let $c\in|N|$. The only rule in $\Pi\cup ch(N)$
with $c$ as the head is $c \ar not\ not\ c$. It follows that $c\in M\cup N$.
Since $|N|\cap|M|=\emptyset$, $c\in N$, a contradiction. Thus, $c\notin |N|$
and there is a rule $c\ar B\in \Pi$ such that $s(B)\subseteq M\cup N$. Let 
$B'$ be what remains when we remove from $B$ all expressions $l$ such that 
$s(l)\in N$. The rule $c\ar B'\in \Pi(N)$ and $s(B')\subseteq M$. It follows 
that $c\in \{a\stt a\ar B\in\Pi(N)  \hbox{ and } s(B')\subseteq M\}$. 

Conversely, let $c\in\{a\stt a\ar B\in\Pi(N)  \hbox{ and } s(B) \subseteq
M\}$. It follows that $c\notin |N|$ and so, $c\notin N$. Moreover, there 
is a rule $c\ar B'\in \Pi(N)$ such that $s(B')\subseteq M$. By the 
definition of $\Pi(N)$, there is a rule $c\ar B\in \Pi$ such that $s(B)
\subseteq M\cup N$. Thus, $c\in\{a\stt a\ar B\in\Pi\cup ch(N) \hbox{ and } 
s(B)\subseteq M\cup N\} \setminus N$.
\end{proof}


Let $N$ be a set of literals. We define $N^{-}=\{a\stt \neg a\in N\}$.

\begin{lemma}\label{lem:w2}
For a  logic program $\Pi$, a consistent set $N$ of literals such that 
$|N|\cap Head(\Pi)=\emptyset$, and a consistent set $M$ of literals 
such that $|M|\cap |N|=\emptyset$, 
$\gus(M\cup N,\Pi\cup ch(N))\setminus N^{-}=\gus(M,\Pi(N))$.
\end{lemma}
\begin{proof}

We note that since the sets $M$ and $N$ are consistent and $|M|\cap
|N|=\emptyset$, $M\cup N$ is consistent. Moreover, we note that to
prove the claim it suffices to show that $U$ is an unfounded set on 
$M\cup N$ w.r.t. $\Pi\cup ch(N)$ if and only if $U\setminus N^{-}$ is 
an unfounded set on $M$ w.r.t. $\Pi(N)$.

\smallskip
\noindent
($\Rightarrow$) Let $a\in U\setminus N^{-}$ and let $D\in Bodies(\Pi(N),a)$.
It follows that  $a\notin |N|$. It also follows that there is a rule 
$a \ar B \in \Pi$ such that for every program literal $l\in B$, 
$\overline{s(l)}\notin N$, and $D$ is obtained by removing from $B$ 
every program literal $l$ such that $s(l)\in N$.

Since  $U$ is an unfounded set on $M\cup N$ w.r.t. $\Pi\cup ch(N)$, it 
follows that $\overline{s(B)}\cap (M\cup N)\not=\emptyset$ or $U\cap B^{+}
\not=\emptyset$. In the first case, since for every program literal $l\in 
B$, $\overline{s(l)}\notin N$, $\overline{s(B)}\cap M\not=\emptyset$ 
follows. Moreover, $D$ differs from $B$ only in program literals $l$ 
such that $s(l)\in N$. Since $|M|\cap|N|=\emptyset$, we have
$\overline{s(D)}\cap M\not=\emptyset$. Thus, let us consider the second
case. Let $a\in U\cap B^{+}$. Since $a\notin|N|$, $a\notin N^{-}$. For the
same reason, $a\notin N$. Thus, $a\in U\setminus N^{-}$ and $a\in D^{+}$.
That is, $(U\setminus N^{-})\cap D^{+} \not=\emptyset$. This proves that
$U\setminus N^{-}$ is an unfounded set on $M$ w.r.t. $\Pi(N)$.
  
\smallskip
\noindent
($\Leftarrow$)
Let $U'$ be any unfounded set on
$M$ w.r.t. $\Pi(N)$. By the definition of an unfounded set,
$U'$ contains no atoms from $|N|$ since they do not appear in $\Pi(N)$.
We show that $U'\cup N^{-}$ is an unfounded set on
$M\cup N$ w.r.t. $\Pi\cup ch(N)$. Let $a$ be any atom in $U'\cup N^{-}$. 

\smallskip
\noindent
Case 1. $a\in N^{-}$. It follows that $a$ occurs in the head of only
one rule in $\Pi\cup ch(N)$ namely, $a\ar \ not\ not\ a$. Since $\neg a\in N$,
$\overline{s(not\ not\ a)}\in N$ and, consequently, 
$\overline{s(not\ not\ a)}\in M\cup N$.

\smallskip
\noindent
Case 2. $a\in U'$. It follows that $a\not\in N$ and so,
$Bodies(\Pi\cup ch(N),a)=Bodies(\Pi,a)$. To complete the
argument it suffices to show that for every body $B\in Bodies(\Pi,a)$,
$\overline{s(B)} \cap (M\cup N)\not=\emptyset$ or $(U'\cup N^{-})\cap 
B^{+}\neq\emptyset$ holds. 

Let $B$ be any body in $Bodies(\Pi,a)$. It follows that $\Pi$ contains 
the rule $a \ar B$. If there is a program literal $l$ in $B$ such that 
$\overline{s(l)} \in N$, then the first condition above holds. Thus, let 
us assume that for every program literal $l\in B$, $\overline{s(l)}
\notin N$. Let $D$ be obtained from $B$ by removing from it every program 
literal $l$ such that $s(l)\in N$. It follows that $a\ar D\in \Pi(N)$. 
Since $U'$ is unfounded on $M$ w.r.t. $\Pi(N)$, there is $l$ in $D$ 
such that $\overline{s(l)}\in M$ or $U'\cap D^{+}\not=\emptyset$. In the 
first case, we have $\overline{s(B)} \cap (M\cup N)\not=\emptyset$. In 
the second case, we have $(U'\cup N^{-}) \cap B^{+}\not=\emptyset$. 
\end{proof}

By $W_\Pi^i(M)$ we will denote the $i$-fold application of the $W_\Pi$
operator on the set $M$ of literals. By convention, we assume that 
$W_\Pi^0(M)=M$. 

\begin{lemma}\label{lem:wieq}
For a normal logic program $\Pi$ and a consistent set $N$ of literals 
such that $|N|\cap Head(\Pi)=\emptyset$,
$$
W^i_{\Pi\cup ch(N)}(N)= W^i_{\Pi(N)}(\emptyset)\cup N.
$$
\end{lemma}
\begin{proof}
We proceed by induction on $i$. For $i=0$, since $N$ is consistent,
we have 
$$W^0_{\Pi\cup ch(N)}(N)=N=\emptyset\cup N=W^0_{\Pi(N)}(\emptyset)\cup N.$$
Let us assume that the identity holds for some $i\geq 0$. We show that it
holds for $i+1$.

Let $M$ denote $W^i_{\Pi(N)}(\emptyset)$. We recall that 
$W^{fix}_{\Pi(N)}(\emptyset)$ is the well-founded model of the normal 
program $\Pi(N)$. Consequently, the sets $W^{fix}_{\Pi(N)}(\emptyset)$ 
and $W^j_{\Pi(N)}(\emptyset)$, $j\geq 0$, are consistent \cite{van91}. 
In particular, $M$ is consistent.
Moreover, since $|N|\cap |W^{fix}_{\Pi(N)}(\emptyset)|=\emptyset$,
the sets $W^{j}_{\Pi(N)}(\emptyset)\cup N$, $j\geq 0$, are consistent, 
too. Thus, we have
\begin{eqnarray*}
W^{i+1}_{\Pi(N)}(\emptyset)\cup N &=& N\cup W_{\Pi(N)}(W^i_{\Pi(N)}(\emptyset))
= N\cup W_{\Pi(N)}(M)\\
&=& N\cup M \cup  \{a\stt a\ar B\in\Pi(N)  \hbox{ and } B\subseteq M \} 
\cup \ol{\gus\big(M,\Pi(N)\big)}.
\end{eqnarray*}
Since $|N|\cap |W^{fix}_{\Pi(N)}(\emptyset)|=\emptyset$, $|M|\cap |N|=
\emptyset$. We also observed that $M$ is consistent. By Lemmas~\ref{lem:w1}
and~\ref{lem:w2} and the fact that $\{\neg a\stt 
a\in N^{-}\}\subseteq N$, we have
\begin{eqnarray*}
W^{i+1}_{\Pi(N)}(\emptyset)\cup N &=& 
N\cup (M \cup \{a\stt a\ar B\in\Pi\cup ch(N)  \hbox{ and }
B\subseteq M\cup N \}\setminus N)\\
& & \cup\ \ol{\gus\big(M\cup N,\Pi\cup ch(N)\big)\setminus N^{-}}\\
&=& N\cup (M \cup \{a\stt a\ar B\in\Pi\cup ch(N)  \hbox{ and }
B\subseteq M\cup N \}\setminus N)\\
& & \cup\ (\ol{\gus\big(M\cup N,\Pi\cup ch(N)\big)}\setminus \{\neg a\stt 
a\in N^{-}\})\\  
&=& N\cup M \cup \{a\stt a\ar B\in\Pi\cup ch(N)  \hbox{ and }
B\subseteq M\cup N \}\\
& & \cup\ \ol{\gus\big(M\cup N,\Pi\cup ch(N)\big)}.
\end{eqnarray*}
Since this last set is consistent, it is equal to $W_{\Pi\cup ch(N)}(M\cup N) = W_{\Pi\cup ch(N)}(W^i_{\Pi(N)}(\emptyset) \cup N)$. Applying the induction
hypothesis, the inductive step follows.
\end{proof}

\prop{prop:defrel}{}
{
For a {\pcid} theory $(F,\Pi)$ such that  $\Pi$ is a normal
program, $M$ is a model of $(F,\Pi)$ if and only if $M$ is a model of
$(F,\Pi)$ according to the definition in~\cite{den00}. 
} 
\begin{proof}
Let $(F,\Pi)$ be a {\pcid} theory. In~\cite{den00}, the authors state
that a consistent and complete (over $At(F\cup\Pi)$) set 
$M$ of literals is a model of $(F,\Pi)$ if 
\begin{itemize}
\item[(i)] $M$ is a model of $F$, and 
\item[(ii)] $M=W^{fix}_{\Pi({M^{\opi}})}(\emptyset)\cup M^{\opi}$.
\end{itemize}

To prove the assertion it is sufficient to show that for any model $M$ 
of~$F$ such that $|M|=At(\Pi\cup F)$, $M=W^{fix}_{\Pi^o}(M^{\opi})$ if 
and only if $M=W^{fix}_{\Pi({M^{\opi}})}(\emptyset)\cup M^{\opi}$.
Let $N=M^{O^\Pi}$. The definitions of $O^\Pi$ and $\Pi^o$ directly
imply that $|N|\cap Head(\Pi)=\emptyset$ and that $\Pi^o=\Pi\cup ch(N)$.
Thus, the property follows from Lemma~\ref{lem:wieq}.
\end{proof}

\subsection{Proofs of Results from Section \ref{sec3}}

\prop{prop:input}{}{
For a logic program $\Pi$ and a set $X$ of atoms,
\begin{enumerate}
\item[(a)] $X\subseteq Head(\Pi)$ and $X$ is an input answer set of $\Pi$
if and only if $X$ is an answer set of $\Pi$.  
\item[(b)] if $(X\setminus Head(\Pi))\cap At(\Pi) = \emptyset$, then 
$X$ is an input answer set of $\Pi$ if and only if $X\cap Head(\Pi)$ is 
an answer set of $\Pi$.
\end{enumerate}
}
\begin{proof}
The proof of part (a) is straightforward and follows directly from the
definition of an input answer set. To prove (b), let us assume first that
$X$ is an input answer set of $\Pi$. By the definition, $X$ is an answer
set of $\Pi\cup(X\setminus Head(\Pi))$. Thus, $X$ is the least model of
the reduct $[\Pi\cup(X\setminus Head(\Pi))]^X$. Clearly, we have 
$[\Pi\cup(X\setminus Head(\Pi))]^X =\Pi^X \cup (X\setminus Head(\Pi))$.
Since $(X\setminus Head(\Pi))\cap At(\Pi)=\emptyset$, 
$\Pi^X=\Pi^{X\cap Head(\Pi)}$. It follows that $X$ is the least model 
of $\Pi^{X\cap Head(\Pi)} \cup (X\setminus Head(\Pi))$. Using again the 
assumption $(X\setminus Head(\Pi))\cap At(\Pi)=\emptyset$, one can show
that $X\cap Head(\Pi)$ is the least model of $\Pi^{X\cap Head(\Pi)}$.  
Thus, $X\cap Head(\Pi)$ is an answer set of $\Pi$

The proof in the other direction is similar. Let us assume that $X\cap 
Head(\Pi)$ is an answer set of $\Pi$. It follows that $X\cap Head(\Pi)$
is the least model of $\Pi^{X\cap Head(\Pi)}$. Since $(X\setminus 
Head(\Pi))\cap At(\Pi)=\emptyset$, $X$ is the least model of 
 $\Pi^{X\cap Head(\Pi)} \cup (X\setminus Head(\Pi))$. Moreover, since 
$\Pi^{X\cap Head(\Pi)} =\Pi^X$, $X$ is the least model of $\Pi^X\cup
(X\setminus Head(\Pi))= [\Pi\cup(X\setminus Head(\Pi))]^X$. Thus,
$X$ is an input answer set of $\Pi$.
\end{proof}

\prop{prop17}{}{
A set of literals $M$ is a model of an SM(ASP) theory $[F,\Pi]$
if and only if $M$ is a model of an SM(ASP) theory $[F,\Pi^o]$.
}
\begin{proof}
Proceeding in each direction, we can assume that $M$ is a complete 
(over $At(F\cup \Pi)$)
and consistent set of literals such that $|M|=|At(F\cup\Pi)|$. It follows
that to prove the assertion it suffices to show that for every such set
$M$, $M^{+}$ is an input answer set of $\Pi$ if and only if $M^{+}$ is
an input answer set of $\Pi^o$. 

We note that $\Pi^o=\Pi\cup\{a\ar not\ not\ a \stt a\in At(F\cup
\Pi)\setminus Head(\Pi)\}$. Thus, $M^{+}\subseteq Head(\Pi)$ and so,
by Proposition \ref{prop:input}, $M^{+}$ is an input answer set of
$\Pi^o$ if and only if $M^{+}$ is an answer set of $\Pi^o$. It follows
that to complete the argument, it suffices to show that under our 
assumptions about $M$, $M^{+}$ is an answer set of $\Pi \cup (M^{+}
\setminus Head(\Pi))$ if and only if $M^{+}$ is an answer set of $\Pi^o$. 
This statement is evident once we observe that the reducts of 
$\Pi \cup (M^{+} \setminus Head(\Pi))$ and $\Pi^o$ with respect to
$M^{+}$ are equal (they are both equal to $\Pi^{M^{+}}\cup (M^{+}\setminus Head(\Pi))$).
\end{proof}

\prop{prop:conn}{}{
For any SM(ASP) theory $[F,\Pi]$ that is $\Pi$-safe, a set
$X$ of atoms is an answer set of $\Pi$ if and only if
$X=M^{+}\cap At(\Pi)$, for some model $M$ of $[F,\Pi]$.
}
\begin{proof}
($\Rightarrow$) Let $X$ be an answer set of $\Pi$. Since $[F,\Pi]$ is 
$\Pi$-safe, there is a model $M$ of $F$ such that $X=M^{+}\cap Head(\Pi)$. 
Moreover, again by the $\Pi$-safety of $[F,\Pi]$, $\{\neg a \stt a\in O_\Pi\}
\subseteq M$. It follows that $X=M^{+}\cap At(\Pi)$ and $(M^{+}\setminus 
Head(\Pi))\cap At(\Pi) =\emptyset$. By Proposition \ref{prop:input}(b), 
$M^{+}$ is an input answer set of $\Pi$.

\smallskip
\noindent
($\Leftarrow$) Let $X=M^{+}\cap At(\Pi)$, where $M$ is a model of $[F,\Pi]$. It follows that
$M$ is a model of $F$. By the $\Pi$-safety of $[F,\Pi]$, we have
$\{\neg a \stt a\in O_\Pi\} \subseteq M$. As above, it follows that 
$(M^{+}\setminus Head(\Pi))\cap At(\Pi) =\emptyset$. Since $M^{+}$ is an 
input answer set of $\Pi$, Proposition \ref{prop:input}(b) implies
that $M^{+}\cap Head(\Pi)$ is an answer set of $\Pi$. From the identity
$(M^{+}\setminus Head(\Pi))\cap At(\Pi) =\emptyset$, it follows that
$M^{+}\cap Head(\Pi) = M^{+}\cap At(\Pi)$. Thus, $X$ is an answer set of $\Pi$. 
\end{proof}

Corollary \ref{cor:conn} follows immediately from Proposition 
\ref{prop:conn}. We omit its proof and move on to Proposition 
\ref{prop:pcidsmasp}. We start by proving two simple auxiliary results.


\begin{lemma}\label{lem:wmodel}
For a logic program $\Pi$, and a consistent and complete set $M$ of 
literals over $At(\Pi)$, if $M=W_\Pi(M)$, then $M$ is a model of $\Pi$.
\end{lemma}
\begin{proof}
It is sufficient to show that for every rule $a\ar B\in \Pi$ if $s(B)
\subseteq M$ then $a\in M$. This follows from the definition of the 
operator $W_\Pi$ and the fact that $M=W_{\Pi}(M)$. 
\end{proof}

\begin{lemma}\label{lem:wunf}
For a logic program $\Pi$ and a consistent and complete set $M$ of 
literals over $At(\Pi)$, if $M=W_{\Pi}(M)$ then $M^{+}$ does not have 
any non-empty subset that is unfounded on $M$ with respect to $\Pi$.
\end{lemma}
\begin{proof}
Let us assume that $U$ is a non-empty subset of $M^{+}$ that is 
unfounded on $M$ with respect to $\Pi$. It follows that $\ol{U}\subseteq M$.
Since $U\not=\emptyset$, $M$ is inconsistent, a contradiction. 
\end{proof}

Next, we recall the following generalization of a well-known
characterization of answer sets in terms of unfounded sets due 
to Leone et al. \citeyear{leo97}. The generalization extended the
characterization to the case of programs with double negation.

\smallskip
\noindent
{\it Theorem on Unfounded Sets}\cite{lee05}\\
For a set $M$ of literals,  $M^{+}$  is an answer set of a program 
$\Pi$ if and only if $M$ is a model of $\Pi$ and $M^{+}$ does not have 
any non-empty subset that is unfounded on $M$ with respect to $\Pi$.

\smallskip
\noindent
\prop{prop:pcidsmasp}{}{
For a total {\pcid} theory $(F,\Pi)$
and  a set~$M$ of literals over
the set $At(F\cup \Pi)$ of atoms, the following conditions are
equivalent:
\begin{itemize}
\item[(a)] $M$ is a model of $(F,\Pi)$
\item[(b)] $M$ is a model of an SM(ASP) theory
$[F,\Pi]$
\item[(c)] $M$ is a model of an SM(ASP) theory
$[Comp(\Pi_{At(\Pi)})\cup F,\Pi]$
\item[(d)] for some model $M'$ of an
SM(ASP) theory \hbox{$[\i{ED-Comp}(\Pi_{At(\Pi)})\cup F,\Pi]$},
$M=M'\cap At(F\cap \Pi)$.
\end{itemize}
}
\begin{proof}
(a)$\Rightarrow$(b)
It is sufficient to show that $M^{+}$ is an input answer set of
$\Pi$, that is, an answer set of $\Pi\cup (M^{+}\setminus Head(\Pi))$.
Since $M$ is a model of the {\pcid} theory $(F,\Pi)$, $M$ is a complete
and consistent set of literals over $At(F\cup\Pi)$ and 
$M=W_{\Pi^o}^{fix}(M^{O^\Pi})$. It follows that $M=W_{\Pi^o}(M)$. Since
$At(\Pi^o)=At(F\cup \Pi)$, by Lemma~\ref{lem:wmodel} it follows that $M$
is a model of $\Pi^o$. Consequently, $M$ is a model of $\Pi\cup (M^{+}
\setminus Head(\Pi)$). By Theorem on Unfounded Sets, it is sufficient to show
that  $M^{+}$ does not have any non-empty subset that is
unfounded on $M$ with respect to $\Pi\cup(M^{+}\setminus Head(\Pi))$. 
For a contradiction, let us assume that there is a nonempty set
$U\subseteq M^{+}$ that is unfounded on $M$ with respect to $\Pi\cup 
(M^{+}\setminus Head(\Pi))$. Let $a\in U$. It follows that $a\in M^{+}$.
If $a\notin Head(\Pi)$, then $a$ is a fact in $\Pi\cup(M^{+}\setminus 
Head(\Pi))$. This is a contradiction with the unfoundedness of $U$. Thus,
$a\in Head(\Pi)$. By the definition of $\Pi^o$, $Bodies(\Pi^o,a)=
Bodies(\Pi,a)$. It follows that for every $B\in Bodies(\Pi^o,a)$,
$\ol{s(B)}\cap M\not=\emptyset$ or $U\cap B^{+}\not=\emptyset$. This shows
that $U$ is unfounded on $M$ with respect to $\Pi^o$.
This
contradicts Lemma~\ref{lem:wunf}.

\smallskip
\noindent
(a)$\Leftarrow$(b) Since $M$ is a model of $[F,\Pi]$, $M$ is a complete 
and consistent set of literals over $At(F\cup \Pi)$. By the assumption, 
$M^{+}$ is an answer set of 
$\Pi'=\Pi\cup (M^{+}\setminus Head(\Pi))$. Since $\Pi'$ and $\Pi$ have
the same reducts with respect to $M^{+}$, $M^{+}$ is an answer set of
$\Pi^o$. 

Since $M^{O^\Pi}\subseteq M$, $W_{\Pi^o}(M^{O^\Pi})\subseteq W_\Pi^o(M)$.
Let $l\in W_\Pi^o(M)$. If $l=a$, where $a$ is an atom in $\Pi^o$, then
there is a rule $a\ar B$ in $\Pi^o$ such that $s(B)\subseteq M$. Since
$M$ is a model of $\Pi^o$ (it is so since $M^{+}$ is an answer set of
$\Pi^o$), $a\in M$. If $l=\neg a$, then $a\in GUS(M,\Pi^o)$. 

Let us 
assume that $a\in M^{+}$ and let us define $U=M^{+}\cap GUS(M,\Pi^o)$.
Clearly, $U\not=\emptyset$ and $U\subseteq GUS(M,\Pi^o)$. 
Let $b\in U$ and let $B\in Bodies(\Pi^o,b)$. Let us assume that $\ol{s(B)}
\cal M=\emptyset$. By the completeness of $M$, $s(B)\subseteq M$. 
Since $b\in GUS(M,\Pi^o)$, there is an element $GUS(M,\Pi^o)
\cap B^{+}\not=\emptyset$. Let us assume that $c\in GUS(M,\Pi^o)
\cap B^{+}$. It follows that $c\in M^{+}$ and so, $c\in U$. Thus,
$U$ is a nonempty set contained in $M^{+}$ and unfounded on $M$ with
respect to $\Pi^o$. By Theorem on Unfounded Sets, this contradicts
the fact that $M^{+}$ is an answer set of $\Pi^o$. it follows that $a\notin
M^{+}$. By the completeness of $M$, $\neg a\in M$. Thus,
$W_\Pi^o(M)\subseteq M$ and, consequently, $W_{\Pi^o}(M^{O^\Pi})\subseteq M$.
By iterating, we obtain that $W^{fix}_{\Pi^o}(M^{O^\Pi})\subseteq M$.
Since $(F,\Pi)$ is total, $W^{fix}_{\Pi^o}(M^{O^\Pi})= M$. Thus, (a)
follows.

\smallskip
\noindent
(b)$\Leftrightarrow$(c) It is sufficient to show that $M$ is a model
of $F$ if and only if $M$ is a model of $Comp(\Pi^o)\cup F$ given 
that  $M^{+}$ is an input answer set of $\Pi$ or, equivalently, that
$M^{+}$ is an answer set of $\Pi\cup M^{+}\setminus Head(\Pi)$. The 
``if'' part is obvious. For the ``only if'' part, we proceed as follows.
First, reasoning as above we observe that $M^{+}$ is an answer set of 
$\Pi^o$. Thus, $M$ is the model of the completion $Comp(\Pi^o)$ and so,
$M$ is a model of  $Comp(\Pi^o)\cup F$, which we needed to show.

\smallskip
\noindent
(b)$\Leftrightarrow$(d) The equivalence follows from the fact that 
$\i{ED-Comp}(\Pi_{At(\Pi)})$ is a conservative extension of $Comp(\Pi_{At(\Pi)})$.
%
\end{proof}

We now proceed to the proof of Proposition~\ref{prop:cm1}. We first recall
a result proved by Lierler \citeyear{lier10} (using a slightly modified
notation).. 

\begin{lemma}[Lemma 4~\cite{lier10}]\label{lem:uset}
For any unfounded set $U$ on a consistent set~$M$ of literals with 
respect to a program~$\Pi$,
and any assignment $N$, if $N\models M$ and $N\cap U\neq
\emptyset$, then $N^{+}$ is not an answer set for $\Pi$.
\end{lemma}

It is well known that for any consistent and complete set $M$ of literals 
over $At(\Pi)$ (\emph{assignment} on $At(\Pi)$), if $M^{+}$ is an answer 
set for a program $\Pi$, then $M$ is a model of $\Pi^{cl}$. The property
has a counterpart for SM(ASP) theories. The proof is straightforward and
we omit it.

\begin{lemma}\label{lem:smasppicl}
For every SM(ASP) theory $[F,\Pi]$, if $M$ is a model of $[F,\Pi]$,
then $M$ is a model of $F\cup \Pi^{cl}$. 
\end{lemma}

Next, we prove the following auxiliary result. 

\begin{lemma}\label{lem:des_cm1}
For every SM(ASP) theory $[F,\Pi]$, every state $M$ other than $\fail$
reachable from $\emptyset$ in ${\smtasp}_{F,\Pi}$, and every model $N$ 
of $[F,\Pi]$, if $N$ satisfies all decision literals in $M$, then $N$ 
satisfies $M$.
\end{lemma} 
\begin{proof}
We proceed by induction on $n=|M|$. The property trivially holds for $n=0$.
Let us assume that the property holds for all states with $k'\leq k$ 
elements that are reachable from $\emptyset$. For the inductive step,
let us consider a state $M=l_1~\dots~l_k$ such that every model $N$ of
$[F,\Pi]$ that satisfies all decision literals $l_j$ with $j\leq j$ 
satisfies $M$. We need to prove that applying any transition rule of 
${\smtasp}_{F,\Pi}$ in the state $l_1~\dots~l_k$, leads to a state 
$M'=l_1~\dots~l_k,~l_{k+1}$ such that if $N$ is a model of $[F,\Pi]$ 
and $N$ satisfies every decision literal $l_j$ with $j\leq k+1$,
then $N$ satisfies $M'$. 

\smallskip
\noindent
{\rup}: By the definition of {\rup}, there is a clause $C\vee l\in 
F\cup \Pi^{cl}$ such that $\ol C \subseteq M$ and $M'= M~l$. Let $N$ be 
any model of~$[F,\Pi]$ that satisfies all decision literals $l_j\in M~l$.
It follows that $N$ satisfies all decision literals in $M$. By the induction
hypothesis, $N\models M$. Since $N\models C\vee l$ and $\ol C \subseteq M$,
Lemma \ref{lem:smasppicl} implies that $N\models l$. 

\smallskip
\noindent
{\rd}: In this case, $M'=M~l^d$ ($l$ is a decision literal). If $N$ is 
a model of the theory $[F,\Pi]$ and it satisfies all decision literals 
in $M'$, then $N$ satisfies $l$ (by the assumption) and $N$ satisfies 
every decision literal in $M$. By the induction hypothesis, the latter
implies that $N\models M$. Thus, $N\models M'$.

\smallskip
\noindent
{\rf}: If this rule is applicable, $M$ has no decision literals and is
inconsistent. If $[F,\Pi]$ has a model $N$, then by the induction 
hypothesis, $N\models M$, a contradiction. It follows that $[F,\Pi]$
has no models and the assertion is trivially true.

\smallskip
\noindent
{\rb}: If this rule is applied, it follows that $M$ has the form 
$P~l^d_{i}~Q$, where $Q$ contains no decision literals, and $M' =
P~\overline{l_{i}}$. Let $N$ be a model of $[F,\Pi]$ such that $N$
satisfies all decision literals in $P\overline{l_{i}}$. It follows 
that $N$ satisfies all decision literals in $P$ and so, by the 
induction hypothesis, $N \models P$. Let us assume that $N\models
\!\!{l_{i}}$. Then, $N$ satisfies all decision literals in $M$ and, 
consequently, $N\models M$, a contradiction as $M$ is inconsistent.
Thus, $N\models \overline{l_{i}}$ and so, $N\models M'$.

\smallskip
\noindent
\runf: If $M'$ is obtained from $M$ by an application of the $\runf$ 
rule, then $M$ is consistent and $M'=M~\neg a$, for some $a\in U$, 
where~$U$ is an unfounded set on $M$ with respect to $\Pi^o$. Let $N$ 
be any model $N$ of $[F,\Pi]$ such that $N$ satisfies all decision 
literals in $M'$. It follows that $N$ satisfies all decision literals 
in $M$ and so, by the inductive hypothesis, $N\models M$. By the 
definition of a model of $[F,\Pi]$, $N^{+}$ is an input answer set of 
$\Pi$. Consequently, $N^{+}$ is an answer set of \hbox{$\Pi\cup(N^{+}
\setminus Head(\Pi))$.} Arguing as as before, we obtain that $N^{+}$ 
is an answer set of $\Pi^o$. By Lemma \ref{lem:uset}, $a\notin N^{+}$,
that is, $N\models \neg a$.
\end{proof}

\prop{prop:cm1}{}
{
For any SM(ASP) theory $[F,\Pi]$,
\begin{itemize}
\item [(a)] graph ${\smtasp}_{F,\Pi}$ is finite and acyclic,
\item [(b)] for any terminal state $M$ of ${\smtasp}_{F,\Pi}$ other
  than {\fail}, $M$ is a model of~$[F,\Pi]$
\item [(c)] {\fail} is reachable from $\emptyset$ in 
${\smtasp}_{F,\Pi}$  if and  only if $[F,\Pi]$ has no models.
\end{itemize}
}
\begin{proof}
Parts~(a) and (c) are proved as in the proof of
Proposition~\ref{prop:dp}~\cite[Proposition~1]{lier10} using Lemma~\ref{lem:des_cm1}. 

\noindent
(b) Let $M$ be a terminal state. It follows that none of the rules are 
applicable. From the fact that {\rd} is not applicable, we derive that
$M$ assigns all literals. Since neither {\rb} nor {\rf} are applicable,  
$M$ is consistent. Since {\rup} is not applicable, it follows that for 
every clause $C\vee a\in F\cup \Pi^{cl}$ if $\ol C \subseteq M$ then 
$a\in M$. Consequently, if $M\models \ol C$ then $M\models a$. Thus,
$M$ is a model of $F\cup \Pi^{cl}$. Consequently, $M$ is a model of $F$.

Next, we show that $M^{+}$ is an input answer set of $\Pi$, that is,
that $M^{+}$ is an answer set of $\Pi\cup (M^{+}\setminus Head(\Pi))$. 
To this end, it is sufficient to show that $M^{+}$ is an answer set
of $\Pi^o$ (we again exploit here the fact that $M^{+}$ is an answer
set of $\Pi\cup (M^{+}\setminus Head(\Pi))$ if and only if  $M^{+}$ is 
an answer set of $\Pi^o$). Since $M$ is a model of $F\cup \Pi^{cl}$,
$M$ is a model of $\Pi^o$. 

Let us assume that $M^{+}$ is not an answer set of $\Pi^o$. By Theorem
on Unfounded Sets, it follows that there is a non-empty unfounded set 
$U$ on $M$ with respect to $\Pi^o$ such that $U\subseteq M^{+}$. 
Then {\runf} can be applied for some $a\in U$. If $\neg a\notin M$,
$M$ is not terminal, a contradiction. Thus, $\neg a\in M$. Since $M$
is consistent, $a\notin M^{+}$, a contradiction (as $U\subseteq M^{+}$). 
It follows that $M^{+}$ is an answer set of $\Pi^o$, as required. 
\end{proof}

\rev{Finally, we provide an argument for Proposition \ref{prop:rel_sm_cm1}.
As in other cases, we start with some lemmas.

\begin{lemma}\label{lem:uset2}
For a logic program~$\Pi$ and a consistent set $M$ of literals, 
if a set $U$ of atoms is unfounded on $M$ with respect to $\Pi$ and
$U\setminus Head(\Pi)=\emptyset$, then $U$ is unfounded on $M$ 
with respect to $\Pi^o$.
\end{lemma}
\begin{proof}
From the definition of an unfounded set it follows that for every $a\in 
U$ and every $B\in Bodies(\Pi,a)$, $M \cap \ol{s(B)}\not=\emptyset$ or 
$U\cap B^{pos} \neq \emptyset$.
 From the construction of $\Pi^o$ from $\Pi$ it follows that
 for every $a\in U$ and every $B\in Bodies(\Pi^o,a)$,
$M\models \neg B$ or $U\cap B^{pos}\neq \emptyset$. Consequently,
$U$ is an  unfounded set on $M$ w.r.t. $\Pi^o$.
\end{proof}

\begin{lemma}\label{lem:singular}
If an edge $M~\lrar~M~a$ in  ${\sm}_\Pi$ justified by the
transition rule 
$$
\begin{array}[t]{ll}
\hbox{{\runf($\sm$)}: }&
M\ 
~\lrar~ 
  M~\neg a 
  \hbox{~ if }
  \left\{ \begin{array}{l}
\hbox{$M$ is consistent, and }\\
\hbox{$a\in U$ for a set $U$ unfounded on $M$ w.r.t. $\Pi$}\hbox{,}\\
  \end{array}\right. \\
\end{array}
$$
is  non-singular 
then 
$U\setminus Head(\Pi)=\emptyset$.
\end{lemma}
\begin{proof}
By contradiction. Assume $U\setminus Head(\Pi)\neq\emptyset.$ The 
transition rule {\rarc} is applicable in $M$. Consequently the edge
$M~\lrar~M~a$ is singular. This contradicts our assumption that
$M~\lrar~M~a$ is a non-singular edge.
\end{proof}

Lierler~\citeyear{lier10} introduced the graph $\cmsm_\Pi$ 
whose terminal nodes corresponded to the models of program's
completion. Furthermore, the graph 
$\sm_\Pi$ is defined by means of  $\cmsm_\Pi$ by extending its set of
transition rules by a rule {\runf}(\sm).

\begin{lemma}[Proposition 3~\cite{lier10}]
\label{lem:gr_eq}
For any program $\Pi$, 
the  graphs 
${\cmsm}_\Pi$ and ${\dps}_{\is{Comp}(\Pi)}$
are equal.
\end{lemma}

\prop{prop:rel_sm_cm1}{}{
For every program $\Pi$, 
the graphs $\sm^{-}_\Pi$ and ${\smtasp^{-}}_{Comp(\Pi),\Pi}$ are equal.
}
\begin{proof}
It is easy to see that the states of the graphs $\sm^{-}_\Pi$ and ${\smtasp^{-}}_{Comp(\Pi),\Pi}$ coincide. 
In view of  Lemma~\ref{lem:gr_eq} it is
sufficient to show that 
there is a non-singular edge $M~\lrar~M'$
in  ${\sm}_\Pi$ justified by 
\runf($\sm$) if and only if
 there is a non-singular edge $M~\lrar~M'$
in   ${\smtasp}_{Comp(\Pi),\Pi}$ justified by the rule {\runf}.

\smallskip
\noindent
($\Rightarrow$)  Consider a non-singular edge $M~\lrar~M~a$
in   ${\sm}_\Pi$ justified by the rule {\runf(\sm)}.
By Lemma~\ref{lem:singular}, $U\setminus Head(\Pi)=\emptyset$.
By Lemma~\ref{lem:uset2}, $U$ is an unfounded set on $M$
w.r.t. $\Pi^o$. Consequently, there is
an edge $M~\lrar~M~a$
in   ${\smtasp}_{Comp(\Pi),\Pi}$ justified by the rule {\runf}.
In view of Lemma~\ref{lem:gr_eq} it follows that 
this edge is non-singular in ${\smtasp}_{Comp(\Pi),\Pi}$.

\smallskip
\noindent
($\Leftarrow$) Consider a non-singular edge $M~\lrar~M~a$
in   ${\smtasp}_{Comp(\Pi),\Pi}$ justified by the rule {\runf}. Then 
there is  a set $U$ unfounded on $M$ w.r.t. $\Pi^o$ such that $a\in U$.
 By Lemma~\ref{lem:uset1} with $Y$ as $\emptyset$, $U$ is also
unfounded on $M$ w.r.t. $\Pi$. It follows that 
 there is an edge $M~\lrar~M~a$ in 
 ${\sm}_\Pi$ justified by {\runf(\sm)}.
In view of Lemma~\ref{lem:gr_eq} it follows that 
this edge is non-singular in  ${\sm}_\Pi$.
\end{proof} 
}{Finally, we sketch a proof for Proposition \ref{prop:rel_sm_cm1}
\prop{prop:rel_sm_cm1}{}{
For every program $\Pi$,
the graphs $\sm^{-}_\Pi$ and ${\smtasp^{-}}_{Comp(\Pi),\Pi}$ are equal.
}
\begin{proof} Sketch: First we show that
 the states of the graphs $\sm^{-}_\Pi$ and
 ${\smtasp^{-}}_{Comp(\Pi),\Pi}$ coincide. 
In view of Proposition 3 stated and proved by Lierler~\cite{lier10}
it is sufficient to show that there is a non-singular edge $M~\lrar~M'$
in  ${\sm}_\Pi$ justified by the transition {\runf} (defined for $\sm$) if and only if there is a non-singular edge $M~\lrar~M'$ in ${\smtasp}_{Comp(\Pi),\Pi}$ justified by {\runf} (defined for
 \smtasp). We conclude by proving the last statement.
\end{proof}

}

\subsection{Proof of Proposition~\ref{prop:clasp}}

We first extend Lemma \ref{lem:des_cm1} to the ``learning'' version of
the graph  ${\smtasp}_{F,\Pi}$.

\begin{lemma}\label{lem:des_clasp}
For every SM(ASP) theory $[F,\Pi]$, every state $M\dbar\Gamma$
reachable from $\emptyset\dbar\emptyset$ in ${\smtasp}_{F,\Pi}$, and 
every model $N$ of $[F,\Pi]$, if $N$ satisfies all decision literals in 
$M$, then $N$ satisfies $M$.
\end{lemma} 
%
\begin{proof}
The proof is by induction on $n=|M|$ and proceeds similarly as that of 
Lemma \ref{lem:des_cm1}. In particular, the property trivially holds 
for $n=0$. Let us assume that the property holds for all states $M\dbar\Gamma$,
where $|M| \leq k$, that are reachable from $\emptyset\dbar\emptyset$. For the 
inductive step, let us consider a state $M\dbar\Gamma$, with 
$M=l_1~\dots~l_k$, such that every model $N$ of
$[F,\Pi]$ that satisfies all decision literals $l_j$ with $j\leq k$
satisfies $M$. We need to prove that applying any transition rule of
${\smtasp}_{F,\Pi}$ in the state $M\dbar\Gamma$, leads to a state 
$M'\dbar\Gamma'$, where $M'=M~l_{k+1}$, such that if $N$ is a model of 
$[F,\Pi]$ and $N$ satisfies every decision literal $l_j$ with $j\leq 
k+1$, then $N$ satisfies $M'$.

The rules {\rd}, {\rf} and {\runf} can be dealt with as before 
(with only minor notational adjustments to account for extended states). 
Thus, we move on to the rules {\rupl}, {\rbj}, and {\rl}.

\smallskip
\noindent
{\rupl}: We recall that $\Gamma$ is a set of clauses entailed by $F$ 
and $\Pi$. In other words, any model of~$[F,\Pi]$ is also a model of 
$\Gamma$. We now proceed as in the case of the rule {\rup} in the proof
of Proposition \ref{lem:des_cm1} with $F\cup \Pi^{cl}$ replaced by 
$F\cup \Pi^{cl}\cup \Gamma$.

\smallskip
\noindent
{\rbj}: The argument is similar to that used in the case of the 
transition rule {\rb} in the proof of Lemma~\ref{lem:des_cm1}.

\smallskip
\noindent
{\rl}: This case is trivially true.
\end{proof}

We now recall several concepts we will need in the proofs. Given a set
$A$ of atoms, we define $Bodies(\Pi,A) = \bigcup_{a\in A}~Bodies(\Pi,a)$.
Let $\Pi$ be a  program and $Y$ a set of atoms. We call the formula
\begin{equation}
\label{eqMT:new}
\bigvee_{a\in Y}{ a} \rightarrow \bigvee\{B\stt B\in Bodies(\Pi,Y)
\mbox{ and } B^{pos}\cap Y=\emptyset\}
\end{equation}
the \emph{loop formula} for $Y$ \cite{lin04}. We 
can rewrite the loop formula (\ref{eqMT:new}) as the disjunction
\beq
(\bigwedge_{a\in Y}{\neg a}) \vee \bigvee\{ B\stt B\in Bodies(\Pi,Y) \mbox{ and } B^{pos}\cap Y=\emptyset\} \hbox{.}
\eeq{e:lfcl}

The \emph{Main Theorem} in~\cite{lee05} implies the following property
loop formulas. In its statement we refer to the concept of a program 
entailing a formula. The notion is defined as follows. A program $\Pi$
entails a formula $F$ (over the set of atoms in $\Pi$) if for every 
interpretation $M$ (over the set of atoms in $\Pi$) such that $M^{+}$ is 
an answer set of $\Pi$, $M$ is a model of $F$.

\begin{lemma}[Lemma on Loop Formulas]
For every program $\Pi$ and every set $Y$ of atoms, $Y\subseteq At(\Pi)$,
$\Pi$ entails the loop formula~(\ref{e:lfcl}) for $Y$.
\end{lemma}


For an SM(ASP) theory $[F,\Pi]$ and a list  $P~l~Q$ of literals, 
we say that a clause $C\vee l$ is a \emph{reason} for $l$ to be in 
$P~l~Q$ with respect to $[F,\Pi]$ if 
\begin{enumerate}
\item $P\models\neg C$, and  
\item $F,\Pi^o\models C\vee l$. 
\end{enumerate}

\begin{lemma}\label{prop:ereason}
Let $[F,\Pi]$ be an SM(ASP) theory. For every state $M\dbar\Gamma$ reachable
from $\emptyset||\emptyset$ in the graph $\smtaspl_{F,\Pi}$, every literal 
$l$ in $M$ is either a decision literal or has a reason to be in $M$ with
respect to $[F,\Pi]$. 
\end{lemma}
\begin{proof}
We proceed by induction on the length of a path from 
$\emptyset||\emptyset$ to $M||\Gamma$ in the graph $\smtaspl_{F,\Pi}$.
Since the property trivially holds in the initial state 
$\emptyset||\emptyset$, we only need to prove that every transition rule
of $\smtaspl_{F,\Pi}$ preserves it.

Let us consider an edge $M\dbar\Gamma\lrar M'\dbar\Gamma'$, where $M$ 
is a sequence $l_1~\dots~l_k$
such that every $l_i$, $1\leq i\leq k$, is either a decision literal or 
has a reason to be in $M$ with respect to $[F,\Pi]$. It is evident that
transition rules {\rbj}, {\rd}, {\rl}, and {\rf} preserve the property
(the last one trivially, as {\fail} contains no literals).

\smallskip
\noindent
{\rupl}: The edge $M\dbar\Gamma\lrar M'\dbar\Gamma'$ is justified by the
rule {\rupl}. That is, there is a clause $C\vee l\in F \cup \Pi^{cl} \cup
\Gamma$ such that $\ol C \subseteq M$ and $M'=Ml$. By the inductive 
hypothesis, the property holds for every literal in~$M$. We now show that 
a clause $C\vee l$ is a reason for $l$ to be in $M~l$. By the 
applicability conditions of {\rupl}, $\ol C \subseteq M$. Consequently, 
$M\models \ol C$. It remains to show that $F,\Pi^o\models C\vee l$.  

\smallskip
\noindent
Case 1. $C\vee l\in F$. Then, clearly, $F\models C\vee l$ and, consequently,
$F,\Pi^o\models C\vee l$.

\smallskip
\noindent
Case 2. $C\vee l\in \Pi^{cl}$. Since $\Pi^{cl}\subseteq (\Pi^o)^cl$,
 $C\vee l\in (\Pi^o)^{cl}$. Let $M$ be a model of $[F,\Pi^o]$. It follows
that $M^{+}$ is an answer set of $\Pi^o$. Thus, $M\models (\Pi^o)^{cl}$
and so, $M\models C\vee l$. Thus, $F,\Pi^o\models C\vee l$.

\smallskip
\noindent
Case 3. $C \vee l\in \Gamma$. We recall that $F,\Pi^o\models \Gamma$ by the
definition of an augmented state. Consequently, $F, \Pi^o\models C\vee l$.

\smallskip
\noindent
{\runf}: We have that $M$ is consistent, and that there is an unfounded 
set $U$ on $M$ with respect to $\Pi^o$ and $a\in U$ such that $M'=M\neg a$.
By the inductive hypothesis, the property holds for every literal in~$M$. 
We need to show that $\neg a$ has a reason to be in $M\neg a$ with respect
to $[F,\Pi]$. 

Let $B\in Bodies(\Pi^o,U)$ be such that $U\cap B^{pos}=\emptyset$. By 
the definition of an unfounded set, it follows that $\ol{s(B)}\cap M
\neq \emptyset$. Consequently, $s(B)$ contains a literal from $\ol{M}$.
We pick an arbitrary one and call it $f(B)$. The clause
\beq
C = \neg a\vee\bigvee\{f(B)\stt B\in Bodies(\Pi^o,U) \mbox{ and } 
B^{pos}\cap U=\emptyset\},
\eeq{e:runf}
is a reason for $\neg a$ to be in $M\neg a$ with respect to $[F,\Pi]$.

First, by the choice of $f(B)$, for every $B\in Bodies(\Pi^o,U) \mbox{ and }
B^{pos}\cap U=\emptyset$, $\ol{f(B)}\in M$. Consequently,
\beq
M\models \neg \bigvee\{f(B)\stt B\in Bodies(\Pi^o,U) \mbox{ and }
B^{pos}\cap U=\emptyset\}.
\eeq{dummy}
Second, since $f(B)\in B$, the loop formula 
\beq
(\bigwedge_{u\in U}{\neg u}) \vee \bigvee\{B \stt B\in Bodies(\Pi,U) \mbox{ and } B^{pos}\cap U=\emptyset\}
\eeq{dummy2}
entails $C$. By Lemma \emph{on Loop Formulas}, it follows 
that $\Pi^o$ entails $C$. Consequently, $F,\Pi^o\models C$.
\end{proof}

For a list $M$ of literals, by $consistent(M)$ we denote the 
longest consistent
prefix of $M$. For example, $consistent(a~b~c~\neg b~d)=a~b~c$.   
A clause~$C$ is \emph{conflicting} on a list $M$ of literals with respect to
an SM(ASP) theory $[F,\Pi]$ if $consistent(M)\models \neg C$ and 
$F,\Pi^o\models C$. 

For a state $M||\Gamma$ reachable from $\emptyset||\emptyset$ in 
$\smtaspl_{F,\Pi}$, by $r_M$ we denote a function that maps every 
non-decision literal in $M$ to its reason to be in $M$ (with respect 
to $[F,\Pi]$). By $\boldr_M$ we denote the set consisting of the 
clauses $r_M(l)$, for each non-decision literal $l\in consistent(M)$.

A \emph{resolution derivation of a clause $C$ from a sequence
of clauses $C_1,\dots,C_m$}
is a sequence
$C_1,\dots,C_m,\dots,C_n$, where $C\equiv C_l$ for some $l\leq n$, 
and each clause $C_i$ in the sequence is either a clause from $C_1,
\dots,C_m$ or is derived by applying the resolution rule to clauses
$C_j$ and $C_k$, where $j,k<i$ (we call such clauses \emph{derived}).
We say that a clause~$C$ is \emph{derived by a resolution derivation}
from a sequence of clauses $C_1,\dots,C_m$ if there is a resolution 
derivation of a clause $C$ from $C_1,\dots,C_m$.

\begin{lemma}\label{lem:conflict}
Let $[F,\Pi]$ be an SM(ASP) theory, $M||\Gamma$ a state in the graph 
${\smtasp}_{F,\Pi}$ such that $M$ is inconsistent, and $C_1$ a clause
in $\boldr_M$. If clause $C_2$ is conflicting on $M$ with respect
to $[F,\Pi]$, then every clause $C$ derived from $C_1$ and $C_2$
is also a conflicting clause on $M$ with respect to $[F,\Pi]$.
\end{lemma}
\begin{proof}
Let us assume that $C$ is derived from $C_1$ and $C_2$ by resolving on 
some literal $l\in C_1$. Then, $C_2$ is of the form $\ol l \vee C_2'$.

From the fact that $C_1\in \boldr_M$, it follows that $F,\Pi^o\models C_1$
and that $C_1$ has the form $c_1\vee C'_1$, where 
$consistent(M)\models \neg C'_1$. 
Since $C_2$ is conflicting,
$consistent(M)\models \neg C_2$ and  $F,\Pi^o\models C_2$.
By the consistency of $consistent(M)$, there is no literal in $C'_1$
such that its complement occurs in $C_2$. Therefore $l=c_1$ and, consequently,
$C=C'_1\vee C'_2$. It follows that $consistent(M) \models \neg C$. Moreover,
since $F,\Pi^o\models C_1$ and $F,\Pi^o\models C_2$ and $C$ results from
$C_1$ and $C_2$ by resolution, $F,\Pi^o\models C$. 
\end{proof}

For an SM(ASP) theory $[F,\Pi]$ and a node $M||\Gamma$ in ${\smtasp}_{F,\Pi}$,
a resolution derivation $C_1,\dots,C_n$ is {\sl trivial} on $M$ with respect 
to $[F,\Pi]$\footnote{This
  definition is related to the definition of a {\sl trivial}
  resolution derivation \cite{bea04}.}
if
\begin{itemize}
\item[(1)] $\{C_1,\dots,C_i\}=\boldr_M$
\item[(2)] $C_{i+1}$ is  a conflicting clause on $M$ with respect to
$[F,\Pi]$
\item[(3)] $C_{j}$, $j>i+1$, is derived from $C_{j-1}$ and a clause 
$C_k$, where $k\leq i$ (that is, $C_k\in \boldr_M$), by resolving on 
some non-decision literal of $consistent(M)$.
\end{itemize}

For a record $M_0~l_1~M_1\dots~l_k~M_k$, where $l_i$ are all the
decision literals of the record, we say that  
the literals of $l_i~M_i$ {\sl belong to a decision level} $i$.
For a state $M~l~M'~l'~M''$, we say that $l$ is \emph{older} than $l'$.
We say that a state is a \emph{backjump} state if it is inconsistent, 
contains a decision literal, and is reachable from $\emptyset||\emptyset$
in $\smtaspl_{F,\Pi}$.

\begin{lemma}\label{lem:bj}
For every SM(ASP) theory $[F,\Pi]$, the transition rule {\rbj} is applicable
in every backjump state in ${\smtasp}_{F,\Pi}$.
\end{lemma}
\begin{proof}
Let $M\dbar\Gamma$ be a backjump state in ${\smtasp}_{F,\Pi}$. We will 
show that $M$ has the form $P~l^\dec~Q$ and that there is a literal $l'$
that has a reason to be in $P~l'$ with respect to $[F,\Pi]$. 

Since $M||\Gamma$ is a backjump state, it follows that $M$ has the 
form $consistent(M)~l~~N$. It is clear that $l$ is not a decision 
literal (otherwise $consistent(M)~l$ would be consistent). By Lemma
\ref{prop:ereason}, there is a reason, say $R$ for $l$ to be in $M$.
We denote this reason by $R$. Since $consistent(M)~l$ is inconsistent,
$\ol l \in consistent(M)$. This observation and the definition of a reason
imply that $consistent(M)\models \neg R$. Moreover, since $F,\Pi^o\models
R$ (as $R$ is a reason), $R$ is a conflicting clause. 

Let $dec$  be the largest of the decision levels of the complements
of the literals in $R$ (each of them occurs in $consistent(M)$). Let
$D$ be the set of all non-decision literals in $consistent(M)$. By
$D^{dec}$ we denote a subset of~$D$ that contains all the literals 
that belong to decision level $dec$.

It is clear that $C_1,\dots,C_i,C_{i+1}$, where $\{C_1,\dots C_i\}=\boldr_M$ 
and $C_{i+1}=R$, is a trivial resolution derivation with respect to $M$ and
$consistent(M)\models \neg C_{i+1}$. Let us consider a trivial resolution 
derivation with respect to $M$ of the form $C_1,\ldots,C_i,C_{i+1},\ldots, 
C_n$, where $n\geq i+1$ and $consistent(M) \models \neg C_n$. Let us assume
that there is a literal $l\in D$ such that $\ol l$ in $C_n$. It follows that
$C_n=\ol l\vee C_n'$, for some clause $C'_n$. 

Since $l\in D$ (is a non-decision literal in $consistent(M)$), the set $R_M$ 
contains the clause $r_M(l)$, which is a reason for $l$ to be in $M$. The 
clause $r_M(l)$ is of the form $l\vee l_1\vee\ldots \vee l_m$, where literals 
$\ol{l_1},\dots, \ol{l_m}$ are older than $l$ and $consistent(M)\models
\neg(l_1\vee\ldots \vee l_m)$. Resolving $C_n$ and $r_M(l)$ yields the
clause $C_{n+1}=C_n'\vee l_1\vee\ldots \vee l_m$. Clearly, $C_1,\ldots,
C_{n+1}$ is a trivial resolution derivation with respect to $M$ and
$conistent(M)\models \neg C_{n+1}$.
 
If we apply this construction selecting at each step a non-decision
literal $l\in D^{dec}$ such that $\ol l \in R$, then at some point we
obtain a clause $C_n$ that contains exactly one literal whose complement
belongs to decision level $dec$ (the reason is that in each step of the
construction, the literal with respect we perform the resolution is
replaced by older ones).

By Lemma~\ref{lem:conflict}, the clause $C=C_n$ is conflicting
on $M$ with respect to $[F,\Pi]$, that is, $consistent(M)\models \neg C$
and $F,\Pi^o\models C$. By the construction, $C=l'\vee C'$, where~$l'$
is the only literal whose complement belongs to the decision level $dec$
and the complements of all literals in $C'$ belong to lower decision 
levels. 

\smallskip
\noindent
Case 1. $dec=0$. Since for every literal $l\in C'$, the decision level
of $\ol l$ is strictly lower than $dec$, $C'=\bot$. Since $M||\Gamma$ 
is a backjump state, $M$ contains a decision literal. Then $M$ can be 
written as $P~l^\dec~Q$, where $P$ contains no decision literals (in 
other words $P$ consists of all literals in $consistent(M)$ of decision 
level $dec=0$) and $\ol {l'}\in P$. Clearly, $P\models \neg C'$ 
(as $C'=\bot$). Since $F,\Pi^o\models C$($=l'\vee C'$), $C$ is a reason
for $l'$ to be in $P~l'$.

\smallskip
\noindent
Case 2.  $dec\geq 1$. Let $l$ be the decision literal in $M$ that starts
the decision level $dec$. Then, $M$ can be written as $P~l^\dec~Q$.  
By the construction of the clause $C$, the complement of every literal 
in $C'$ belongs to a decision level smaller than $dec$, that is, to $P$.
It follows that $P\models \neg C'$. Thus, as before, we conclude that
$C$ is a reason for $l'$ to be in $P~l'$.
\end{proof}

\prop{prop:clasp}{}{
For any SM(ASP) theory $[F,\Pi]$,
\begin{itemize}
\item [(a)] 
every path in $\smtaspl_{F,\Pi}$ contains only finitely many edges
justified by basic transition rules,
\item [(b)] for any  semi-terminal state $M||\Gamma$  of
  $\smtaspl_{F,\Pi}$ reachable from $\emptyset||\emptyset$, $M$ is a
  model of $[F,\Pi]$,
\item [(c)] {\fail} is reachable from $\emptyset||\emptyset$ in
  $\smtaspl_{F,\Pi}$ if and  only if  $[F,\Pi]$ has no models.
\end{itemize}
}
\begin{proof}
Part~(a) is proved as in the proof of 
\underline{Proposition~13$^\uparrow$~\cite{lierphd}} (we preserve the
notation used in that work). 

\smallskip
\noindent
(b)  Let $M||G$ be a semi-terminal state  reachable from
$\emptyset||\emptyset$ (that is, none of the basic rules are applicable.)
Since {\rd} is not applicable, $M$ assigns all literals.  
Next, $M$ is consistent. Indeed, if $M$ were inconsistent then, since 
{\rf} is not applicable, $M$ would contain a decision literal. Consequently, 
$M||\Gamma$ would be a backjump state. By Lemma~\ref{lem:bj}, the transition 
rule {\rbj} would be applicable in~$M||\Gamma$, contradicting our assumption 
that $M||\Gamma$ is semi-terminal. We now proceed as in the proof of 
Proposition~\ref{prop:cm1}~(b) to show $M$ is a model of $F$ and $M^{+}$ is 
an input answer set of $\Pi$. 

\smallskip
\noindent
(c) If {\fail} is reachable from $\emptyset||\emptyset$ in  $\smtaspl_{F,\Pi}$,
then there is a state $M||\Gamma$ reachable from $\emptyset||\emptyset$ in 
$\smtaspl_{F,\Pi}$ such that there is an edge between $M||\Gamma$ and {\fail}. 
By the definition of  $\smtaspl_{F,\Pi}$, this edge is due to the transition 
rule {\rf}. Thus, $M$ is inconsistent and contains no decision literals. By 
Lemma~\ref{lem:des_clasp}, every model~$N$ of~$[F,\Pi]$ satisfies~$M$.
Since $M$ is inconsistent, $[F,\Pi]$ has no models.

Conversely, if $[F,\Pi]$ has no models, let us consider a maximal path in 
$\smtaspl_{F,\Pi}$ starting in $\emptyset||\emptyset$ and consisting of 
basic transition rules. By (a), it follows that such a path is finite and
ends in a semi-terminal state. By (b), this semi-terminal must be {\fail},
because $[F,\Pi]$ has no models. 
\end{proof}

\subsection{Proofs of Results from Section \ref{sec:6}}

\prop{pr:eqt}{}{
For a total {\pcid} theory~$(F,\Pi)$ and a consistent and complete (over
$At(F\cup \Pi)$) set $M$ of literals, $M$ is a model of $(F,\Pi)$ if and 
only if $M^{+}$ is an answer set of $\pi(F,\Pi)$.
}
\begin{proof}
By Proposition \ref{prop:pcidsmasp}, it is enough to prove that $M$ is 
a model of the SM(ASP) theory $[F,\Pi]$ if and only if $M^{+}$ is an 
answer set of $\pi(F,\Pi)$. By the definition of $\pi(F,\Pi)$, $M^{+}$ 
is an answer set of $\pi(F,\Pi)$ if and only if $M^{+}$ is an answer 
set of $\Pi^o$ and a model of $F$. Since $M^{+}$ is a subset of 
$Head(\Pi^o)$ (since $Head(\Pi^o)=At(F\cup\Pi)$), Proposition 
\ref{prop:input}(a) implies that $M^{+}$ is an answer set of $\Pi^o$ 
if and only if $M^{+}$ is an input answer set of $\Pi^{o}$. It follows 
that $M^{+}$ is an answer set of $\pi(F,\Pi)$ if and only if $M$ is a 
model of the SM(ASP) theory $[F,\Pi^o]$. The assertion follows now from 
Proposition \ref{prop17}. 
\end{proof}

\prop{prop:rel3}{}{
For a {\pcid} theory $(F,\Pi)$, we have
 $${\smtaspl}_{{\is{ED-Comp}(\Pi^o)}\cup F,\Pi^o}=
 {\smtaspl}_{{\is{ED-Comp}(\pi(F,\Pi))},\pi(F,\Pi)}\hbox{.}$$
}
\begin{proof}
We recall that $\pi(F,\Pi)=F^r\cup\Pi^o$. From the construction of 
$ED\mbox{-}Comp$, it is easy to see that $$ED\mbox{-}Comp(\Pi^o)\cup F
= ED\mbox{-}Comp(\pi(F,\Pi)).$$ 
Furthermore, from the definition of an unfounded set it follows that
for any consistent set $M$ of literals and a set $U$ of atoms, 
$U$ is unfounded on $M$ with respect to $\Pi^o$ if and only if  
$U$ is unfounded on $M$ with respect to $\pi(F,\Pi)$. 
\end{proof}

\end{document}